\PassOptionsToPackage{table}{xcolor}
\RequirePackage[T1]{fontenc}

\documentclass{article} %
\usepackage{iclr2027_conference,times}

\usepackage{amsmath,amsfonts,bm}

\def\1{\bm{1}}

\DeclareMathAlphabet{\mathsfit}{\encodingdefault}{\sfdefault}{m}{sl}
\SetMathAlphabet{\mathsfit}{bold}{\encodingdefault}{\sfdefault}{bx}{n}

\newcommand{\E}{\mathbb{E}}

\newcommand{\KL}{D_{\mathrm{KL}}}
\newcommand{\Var}{\mathrm{Var}}

\usepackage{enumitem}
\usepackage{hyperref}
\usepackage{url}
\usepackage{graphicx,booktabs,multirow,makecell}
\usepackage{capt-of}
\usepackage[table]{xcolor}
\usepackage{float}
\usepackage{amsthm}

\usepackage[most]{tcolorbox}

\newcommand{\sg}{\operatorname{sg}}
\newcommand{\V}{\mathcal{V}}

\newtheorem{proposition}{Proposition}

\def\eqref#1{Eq.~(\ref{#1})}

\title{%
LEGO-OPD: Factorized Teacher Composition for Multimodal On-Policy Distillation }

\author{%
Jaeyun Shin\textsuperscript{*},
Hangeol Chang\textsuperscript{*} \&
Jong Chul Ye\textsuperscript{}\\
{\normalfont Korea Advanced Institute of Science and Technology (KAIST)}\\
{\normalfont\small
\textsuperscript{*}Equal contribution.}
}

\iclrfinalcopy %
\begin{document}

\maketitle
\lhead{Preprint.}

\begin{abstract}

Multimodal on-policy distillation (OPD) aims to improve visual grounding while preserving the strong reasoning capabilities of language models.
Recent multi-teacher approaches combine LLM and VLM teachers to provide complementary supervision.
However, directly using a VLM's full predictive distribution entangles its visual grounding signal with its own language prior, preventing the grounding information from being transferred independently.
Conversely, increasing the strength of visual supervision can improve perception but may overemphasize visual evidence and degrade language reasoning.
To address this trade-off, we introduce \textbf{LEGO-OPD}, which selectively composes factors from a \textbf{L}anguage \textbf{E}xpert and a \textbf{G}rounding expert into \textbf{O}ne teacher distribution for multimodal OPD.
Under a generalized Bayesian formulation, the language expert provides a prior over candidate tokens, while the grounding expert contributes a visual likelihood that updates this prior, rather than transferring its complete predictive distribution.
This factorized composition allows language reasoning and visual grounding to be controlled independently.
We further introduce adaptive calibration to determine how strongly the visual likelihood should update the language prior at each decoding prefix.
Specifically, LEGO-OPD uses the grounding expert's image-induced prediction shift as a prefix-dependent reference, preventing both insufficient and excessive visual supervision.
Experiments with Qwen3 models show that LEGO-OPD consistently outperforms the evaluated single- and multi-teacher OPD baselines on both multimodal and text-only reasoning tasks.
Moreover, it improves the initial student's visual perception while preserving text-only reasoning.

\end{abstract}

\section{Introduction}
\label{sec:introduction}

On-policy distillation (OPD) is an effective post-training framework for multimodal reasoning in vision-language models (VLMs), distilling teacher predictions at student-generated prefixes~\citep{agarwal2024gkd,yin2026hopd}.
Effective multimodal OPD requires reliable visual grounding and strong reasoning over grounded information.

Recent approaches largely follow two directions.
The first introduces additional specialized teachers to provide complementary supervision, for example by combining external large language models (LLMs) for reasoning with VLM teachers for multimodal guidance~\citep{kwaikeyeteam2026keyevl2,yin2026hopd}.
However, directly combining their complete predictive distributions entangles each teacher's language prior and grounding signal, making it difficult to selectively transfer their complementary strengths.
The second direction strengthens the visual component of the distillation signal to improve grounding~\citep{liu2026vaopd,liang2026vcsd,yoon2026decomposedopd}.
While stronger visual supervision can improve perception, excessive emphasis on visual evidence may distort the underlying language prior and degrade reasoning performance.
Together, these limitations highlight the need for a distillation framework that can separately control language reasoning and visual grounding while combining them into a coherent teacher distribution.

In this paper, we introduce LEGO-OPD, which constructs a single teacher distribution by composing complementary factors from independently selected experts.
Our method consists of two components.
First, we factorize the teacher into a language prior and a visual grounding factor.
An external LLM serves as the language expert and provides a prior distribution over candidate tokens given the textual context, while a VLM serves as the grounding expert and evaluates their compatibility with the image.
This composition admits a generalized Bayesian interpretation in which the language distribution acts as a prior and the visual signal acts as a likelihood that updates the prior into a unified teacher distribution.
Second, we introduce source-referenced calibration to determine how strongly the visual factor should modify the language prior at each prefix.
Existing prediction-space composition methods, such as DExperts, proxy-tuning, and visual contrastive decoding, typically rely on fixed composition coefficients~\citep{liu2021dexperts,liu2024proxytuning,leng2023vcd}.
However, a fixed coefficient can induce substantially different prediction shifts depending on the underlying prior and decoding prefix.
We therefore use the grounding expert's \emph{source KL}, which measures the change in its prediction induced by the image, as a prefix-dependent reference for calibrating the visual update.
The scaling coefficient is adapted so that the magnitude of the visual update is normalized relative to this reference, thereby balancing visual grounding against language-prior supervision.
A global multiplier on the reference further controls the desired degree of visual emphasis.
The composed teacher supervises the student through standard reverse-KL on-policy distillation.

Experiments with Qwen3 models show that LEGO-OPD consistently outperforms the evaluated OPD baselines on both multimodal and text-only reasoning tasks.
Further analysis shows that adaptive calibration strengthens visual grounding while preserving text-only reasoning ability.
Ablation studies confirm that both factorized teacher composition and adaptive calibration are essential to achieving these improvements.

Our contributions are summarized as follows.
\begin{itemize}
    \item \textbf{Factorized teacher composition.}
    We compose a language prior and a visual likelihood estimate from independently selected experts into a single OPD teacher through generalized Bayesian updating.

    \item \textbf{Source-referenced calibration.}
    We adaptively determine the visual update's strength at each prefix, with a global scale controlling the overall visual emphasis.

    \item \textbf{Empirical evaluation.}
    We demonstrate improvements over the evaluated OPD baselines in both multimodal and text-only reasoning, and show that adaptive calibration enables stronger visual perception without sacrificing text-only reasoning performance.
\end{itemize}

\begin{figure}[t!]
    \centering
    \vspace{-0.5cm}
    \includegraphics[width=\linewidth]{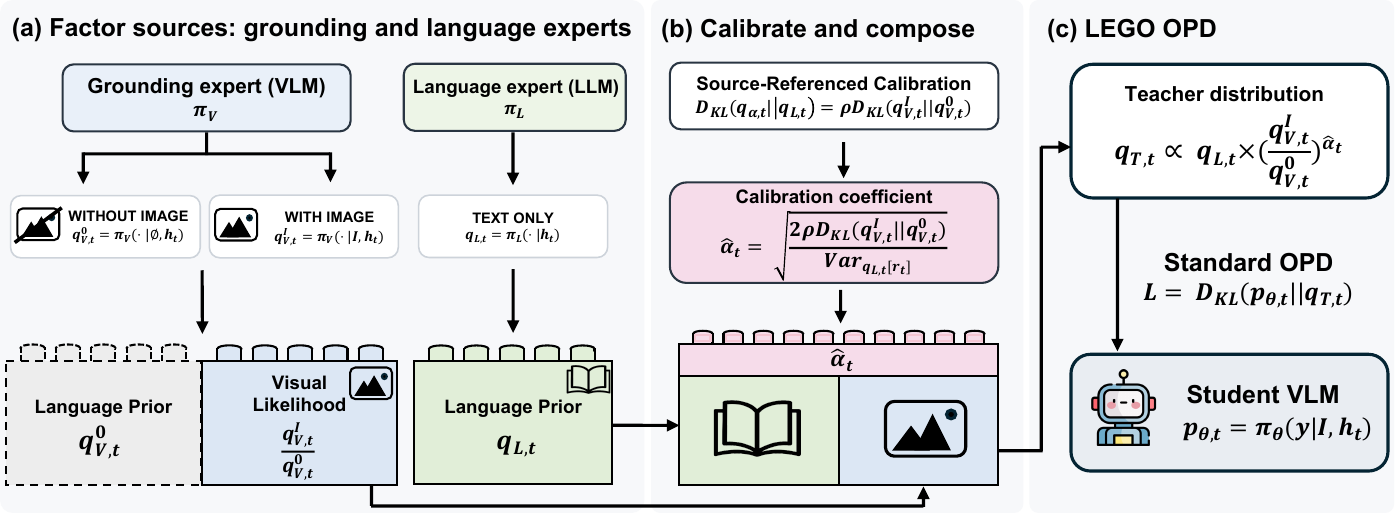}\par
    \caption{
        \textbf{Overview of LEGO-OPD.}
        (a) A grounding expert (VLM) supplies a visual likelihood ratio, while a language expert (LLM) supplies a language prior.
        (b) Source-referenced calibration adjusts the contribution of the visual likelihood ratio.
        (c) The calibrated factors are composed into a single teacher distribution for on-policy distillation of the student VLM.
        }
    \label{fig:method}
\end{figure}

\section{Related Works and Background}
\subsection{Related Work}

\textbf{Heterogeneous Teachers for Multimodal Distillation.}
Recent OPD methods use specialized teachers with complementary strengths,
an approach also explored in multimodal distillation. These methods
differ in the granularity of teacher selection and aggregation. 
For example, M-OPD routes
examples to domain-specific teachers~\citep{ma2026mopd}, and Keye-VL extends
routing across modalities and tasks~\citep{kwaikeyeteam2026keyevl2}. ViGOS
uses differently conditioned self-teachers for perception and reasoning
segments~\citep{wang2026vigos}, while H-OPD combines LLM and VLM predictions
by confidence at each token~\citep{yin2026hopd}.
However, external LLM supervision can require additional
text-only training data or image descriptions
(Appendix~\ref{app:baseline_impl}).
Routing teachers or mixing their complete predictions still entangles
the multimodal teacher's language prior and visual grounding signal,
limiting independent transfer of each expert's strengths.

\textbf{Visual Supervision in Multimodal Distillation.}
A complementary line of multimodal OPD work strengthens the visual
information carried by the supervision signal. VA-OPD uses sensitivity to
visual detail to reweight rollouts and balance supervision across token
groups~\citep{liu2026vaopd}. VCSD refines teacher targets through
original-image and content-erased contrasts, but reports performance
degradation under excessive visual emphasis~\citep{liang2026vcsd}.
Decomposed OPD combines an auxiliary visual objective with a language
preservation loss to mitigate interference from excessive visual
emphasis~\citep{yoon2026decomposedopd}. These findings highlight the
importance of balancing visual and language-prior supervision.

\textbf{Composition in Prediction Space.}
Prediction-space composition offers a way to transfer signals across models.
DExperts uses expert and anti-expert predictions~\citep{liu2021dexperts},
proxy-tuning transfers prediction shifts from smaller tuned
models~\citep{liu2024proxytuning}, and UniR adds learned reasoning-module
logits to a frozen LLM~\citep{kim2026unir}. Visual contrastive decoding
similarly amplifies prediction differences across image
conditions~\citep{leng2023vcd}. However, these methods use fixed composition
coefficients, leaving the choice of suitable composition strengths to
predefined settings or empirical tuning.

LEGO-OPD composes complementary expert factors instead of complete
predictions. Source-referenced calibration selects composition strength
by adapting the visual update at each prefix.

\subsection{Generalized Bayesian Updating}
\label{sec:generalized-bayes}

Generalized Bayesian updating incorporates new information into a prior
without requiring a complete likelihood model for the observations.
The framework of \citet{bissiri2016belief} represents this information
through an observation-dependent loss $L$.
We use the equivalent evidence-score convention $s(v):=-L(v)$,
where higher scores indicate greater support from the observed information.

For a strictly positive prior $p_0$ over a finite candidate space $\V$, a finite
evidence score $s$, and an evidence scale $\alpha\geq0$, the update is
\begin{equation}
q_\alpha^\star
=\arg\max_{q\in\Delta_{\V}}
\left\{\alpha\E_{v\sim q}[s(v)]-\KL(q\|p_0)\right\},
\label{eq:generalized-update}
\end{equation}
where $\Delta_{\V}$ is the set of probability distributions over $\V$. Its unique solution is
\begin{equation}
q_\alpha^\star(v)
=\frac{p_0(v)\exp\{\alpha s(v)\}}
{\sum_{u\in\V}p_0(u)\exp\{\alpha s(u)\}}.
\label{eq:generalized-solution}
\end{equation}
The coefficient $\alpha$ specifies the influence of the evidence relative
to the prior. When $s(v)$ is the log-likelihood of the observations under
a specified probabilistic model and $\alpha=1$, the update recovers that
model's Bayesian posterior. More generally, $s$ can encode information
without specifying a complete sampling distribution.

\subsection{On-Policy Distillation}
\label{sec:opd-background}

On-policy distillation supervises a student policy $\pi_\theta$ at
student-generated prefixes \citep{agarwal2024gkd}.
For $(x,I)\sim\mathcal D$, let $h_t$ contain the system prompt,
question, and response prefix $y_{<t}$.
Write $p_{\theta,t}(v):=\pi_\theta(v\mid I,h_t)$ and $q_t(v)$
for the student and teacher distributions over $\V$ at this prefix.
At update $k$, reverse-KL OPD minimizes
\begin{equation}
\mathcal L_k(\theta)
=
\E_{\substack{
(x,I)\sim\mathcal D\\
y\sim\pi_{\theta_k}(\cdot\mid x,I)
}}
\left[
\frac{1}{|y|}\sum_{t=1}^{|y|}
\KL\!\left(p_{\theta,t}\,\middle\|\,q_t\right)
\right],
\label{eq:opd-loss}
\end{equation}
where $|y|$ is the response length.
Prefixes and teacher targets are fixed during each update.

\section{Method}
\label{sec:method}

\subsection{LEGO-OPD: Multimodal OPD with Factorized Teacher Composition}

\textbf{Notation.}
Let $\pi_L$ and $\pi_V$ denote the frozen policies of an independently
selected language expert (LLM) and grounding expert (VLM), respectively.
For a question--image pair $(x,I)$, let $h_t$ contain the system prompt,
question, and student-generated prefix $y_{<t}$.
Both experts share the token vocabulary $\V$ and are evaluated at
the same textual context. For $v\in\V$, define
$q_{L,t}(v):=\pi_L(v\mid h_t)$,
$q_{V,t}^{I}(v):=\pi_V(v\mid I,h_t)$, and
$q_{V,t}^{0}(v):=\pi_V(v\mid\emptyset,h_t)$,
where $\emptyset$ denotes the no-image condition.
All three distributions are obtained by applying softmax to finite
logits over the finite vocabulary $\V$. %

\textbf{Overall Framework.}
Figure~\ref{fig:method} illustrates our LEGO-OPD  for on-policy VLM distillation using factorized teacher composition.
A grounding expert (VLM) supplies a visual likelihood ratio, while a language expert (LLM) supplies a language prior.
Additionally, source-referenced calibration adjusts the contribution of the visual likelihood ratio, while
the calibrated factors are composed into a single teacher distribution for on-policy distillation of the student VLM.

Specifically, let $z_{L,t}$, $z_{V,t}^{I}$, and $z_{V,t}^{0}$ denote logits whose
softmaxes give $q_{L,t}$, $q_{V,t}^{I}$, and $q_{V,t}^{0}$,
respectively.
Using the adaptive coefficient from
\eqref{eq:local-coefficient}, which will be explained soon,  the composition yields the final
teacher distribution
\begin{equation}
q_{T,t}
:=
q_{\widehat\alpha_t,t}
=
\operatorname{softmax}_{\V}
\left(
z_{L,t}
+
\widehat\alpha_t
\left[z_{V,t}^{I}-z_{V,t}^{0}\right]
\right).
\label{eq:final-teacher}
\end{equation}
Branch-specific log-normalizers cancel in the final softmax, making the
logit- and probability-space compositions equivalent.
For numerical stability, the implementation uses a stabilized residual
and regularized coefficient estimates, with statistics computed over
high-probability tokens to limit sensitivity to tail predictions.
Details are in Appendix~\ref{sec:implementation}.

Finally, let $\pi_\theta$ denote the multimodal student policy, with next-token
distribution $p_{\theta,t}(v):=\pi_\theta(v\mid I,h_t)$ at each
student-generated prefix.
The language expert $\pi_L$ and grounding expert $\pi_V$ provide
the factors used to construct $q_{T,t}$ in
\eqref{eq:final-teacher}.
We use this composed teacher distribution as the target for the
reverse-KL OPD objective in \eqref{eq:opd-loss}, holding
the sampled prefixes and teacher targets fixed during each
student update.

\subsection{LEGO-OPD: Theoretical Background} %

In this section, we provide theoretical background how we come up with the
factored teacher \eqref{eq:final-teacher} as our OPD target in our LEGO-OPD.

\textbf{Bayesian Factorization and Cross-Expert Sourcing.}
For a joint distribution over the next token and image, Bayes' rule
decomposes a multimodal prediction into two complementary factors,
\begin{equation}
P(v\mid I,h_t)
\propto
\underbrace{P(v\mid h_t)}_{\text{language prior}}
\underbrace{P(I\mid v,h_t)}_{\text{visual likelihood}}.
\label{eq:bayes}
\end{equation}
The language prior supports reasoning based on the textual context.
The visual likelihood supports grounding by reweighting tokens
based on their compatibility with the image.
We source the two factors from independently selected experts to draw
on their complementary strengths in reasoning and grounding.
The language expert $\pi_L$ provides the language prior $q_{L,t}$
as its next-token distribution conditioned on the textual context $h_t$.

Although the autoregressive grounding expert $\pi_V$ cannot directly
evaluate $P(I\mid v,h_t)$, Bayes' rule (\eqref{eq:bayes}) gives
\begin{equation}
P(I\mid v,h_t)
\propto
\frac{P(v\mid I,h_t)}{P(v\mid h_t)}.
\label{eq:bayes-ratio}
\end{equation}
To estimate the visual likelihood factor, we use the ratio of the
grounding expert's image-conditioned and image-free predictions,
following Decomposed OPD~\citep{yoon2026decomposedopd}.
The estimate and its logarithm are defined as follows.
\begin{equation}
\widehat\ell_t(v):=\frac{q_{V,t}^{I}(v)}{q_{V,t}^{0}(v)},
\qquad
r_t(v):=\log\widehat\ell_t(v)
=\log q_{V,t}^{I}(v)-\log q_{V,t}^{0}(v).
\label{eq:residual}
\end{equation}

In particular, \(q_{V,t}^{0}\) is obtained by removing the image input rather than marginalizing over images. The ratio in \eqref{eq:residual} therefore need not correspond to the likelihood term of an underlying joint distribution. Instead, it measures the image-induced change in the grounding expert's own predictive distribution and is inherently model-dependent. We thus regard \(\widehat{\ell}_t\) as a proxy for visual evidence, which we refer to as the visual likelihood estimate, and its logarithm \(r_t\) as the visual residual.

\textbf{Generalized Bayesian Teacher Composition.}
We obtain the language prior $q_{L,t}$ and the visual likelihood
estimate $\widehat{\ell}_t$ from independently selected experts.
Their normalized product is therefore not guaranteed to represent
an exact Bayesian posterior. We therefore adopt generalized Bayesian updating, as described in Section~\ref{sec:generalized-bayes}, to update the language prior using the evidence provided by the visual likelihood estimate.

More specifically, we use the visual residual $r_t(v)=\log\widehat\ell_t(v)$ as
the evidence score. Higher values indicate greater relative support
for a candidate token in the presence of the image.
At each student-generated prefix, substituting $p_0=q_{L,t}$ and
$s=r_t$ into \eqref{eq:generalized-update} defines the teacher
distribution for an evidence scale $\alpha_t\geq0$,
\begin{equation}
q_{\alpha,t}:=
\arg\max_{q\in\Delta_{\V}}
\left\{
\alpha_t\E_{v\sim q}[r_t(v)]-\KL(q\|q_{L,t})
\right\}.
\label{eq:teacher-objective}
\end{equation}
The objective balances the expected evidence score against KL
deviation from the language prior.
The unique solution for \eqref{eq:teacher-objective} is
\begin{equation}
q_{\alpha,t}(v)
=
\frac{q_{L,t}(v)\exp\{\alpha_t r_t(v)\}}{Z_t(\alpha)},
\qquad
Z_t(\alpha):=\sum_{u\in\V}q_{L,t}(u)\exp\{\alpha_t r_t(u)\},
\label{eq:family}
\end{equation}
where $Z_t(\alpha)$ is the normalizing constant.
By plugging in \eqref{eq:residual} into \eqref{eq:family}, our teacher assumes the
specific form:
\begin{align}
q_{\alpha,t}(v)
=\frac{1}{Z_t(\alpha)} q_{L,t}(v)\left(\frac{q_{V,t}^{I}(v)}{q_{V,t}^{0}(v)}\right)^{\alpha_t} %
\end{align}
whose logit expression leads to the final form \eqref{eq:final-teacher} of our factored composite teacher.
Accordingly, the generalized Bayesian posterior defines the teacher
distribution by combining the language prior from the language expert
with the visual likelihood from the grounding expert.
Any candidate-independent multiplicative constant in the visual
likelihood cancels after normalization.

The evidence scale $\alpha_t$ determines the strength of the visual
update relative to the language prior.
We next describe how $\alpha_t$ is adaptively selected at each
student-generated prefix.

\textbf{Source-Referenced Calibration.}
We interpret $\alpha_t$ as a calibration coefficient for composing
the language prior and visual likelihood estimate.
Motivated by the limitations of a fixed coefficient illustrated
in Appendix~\ref{app:toy_example}, we use the grounding expert's
image-induced prediction shift as a reference for selecting
$\alpha_t$ at each prefix.
This shift is measured by the \emph{source KL},
$\KL(q_{V,t}^{I}\|q_{V,t}^{0})$.
More specifically, at each prefix, we choose $\alpha_t$ so that the composed teacher's
KL divergence from the language prior matches the scaled source KL:
\begin{equation}
\KL(q_{\alpha_t,t}\|q_{L,t})
=
\rho\,\KL(q_{V,t}^{I}\|q_{V,t}^{0}),
\label{eq:match}
\end{equation}
where $\rho\geq0$ is the KL-budget scale shared across prefixes.
For a nonconstant residual,
$\KL(q_{\alpha,t}\|q_{L,t})$ is continuous and strictly increasing
in $\alpha>0$, so \eqref{eq:match} admits a unique finite
positive solution whenever the requested budget lies within its
attainable range
(Appendix~\ref{sec:existence}).

Solving for this coefficient exactly would require a one-dimensional
root search at every generated position. We therefore locally
approximate $\KL(q_{\alpha,t}\|q_{L,t})$ around $\alpha=0$, where the
composed teacher equals the language prior. A second-order Taylor
expansion gives
\begin{equation}
\KL(q_{\alpha,t}\|q_{L,t})
=
\tfrac12
\Var_{q_{L,t}}[r_t]\alpha^2
+
O(\alpha^3).
\label{eq:quadratic-calibration}
\end{equation}
Matching the quadratic term to the KL budget gives the
approximate coefficient
\begin{equation}
\widehat\alpha_t
:=
\sqrt{
\frac{
2\rho\KL(q_{V,t}^{I}\|q_{V,t}^{0})
}{
\Var_{q_{L,t}}[r_t]
}
}.
\label{eq:local-coefficient}
\end{equation}
For a fixed KL budget, this approximation assigns a smaller coefficient
when the residual varies more strongly under the receiving language
prior. We set $\widehat\alpha_t=0$ when the variance is zero.
Because the approximation is local, the achieved KL may differ from
the requested budget for finite updates.
The derivation and approximation properties are detailed in Appendix~\ref{sec:alpha0}.

\begin{figure}[!t]
\centering
\setlength{\abovecaptionskip}{3pt}
\begin{minipage}[t]{0.49\linewidth}
\centering
\includegraphics[width=\linewidth]{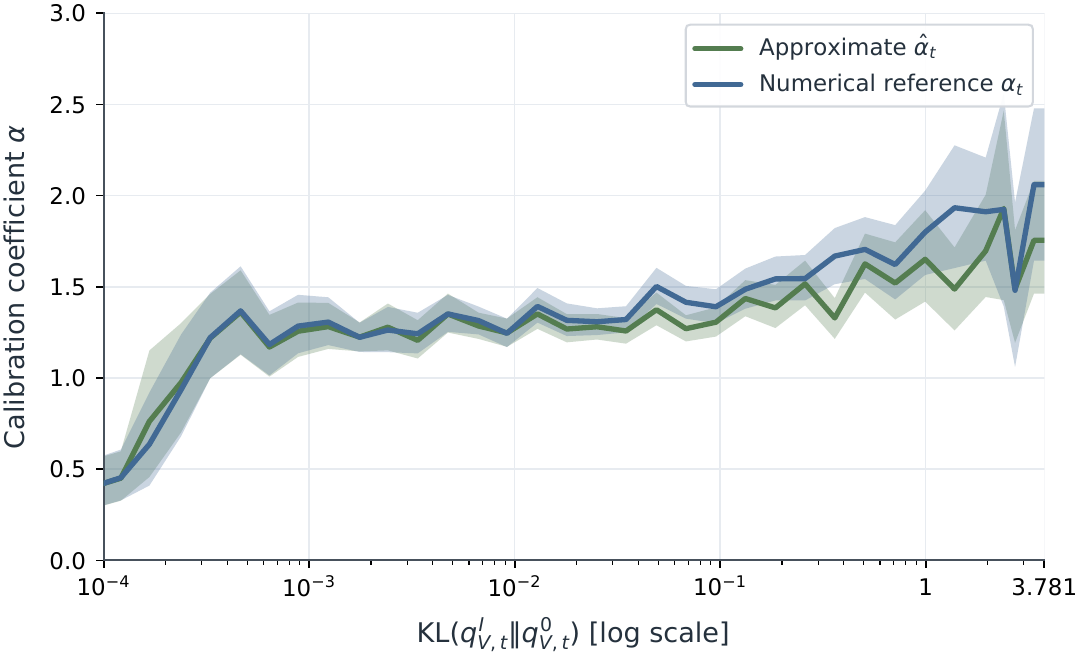}
\par\vspace{-3pt}
{\small (a) Calibration coefficient}
\end{minipage}
\hfill
\begin{minipage}[t]{0.49\linewidth}
\centering
\includegraphics[width=\linewidth]{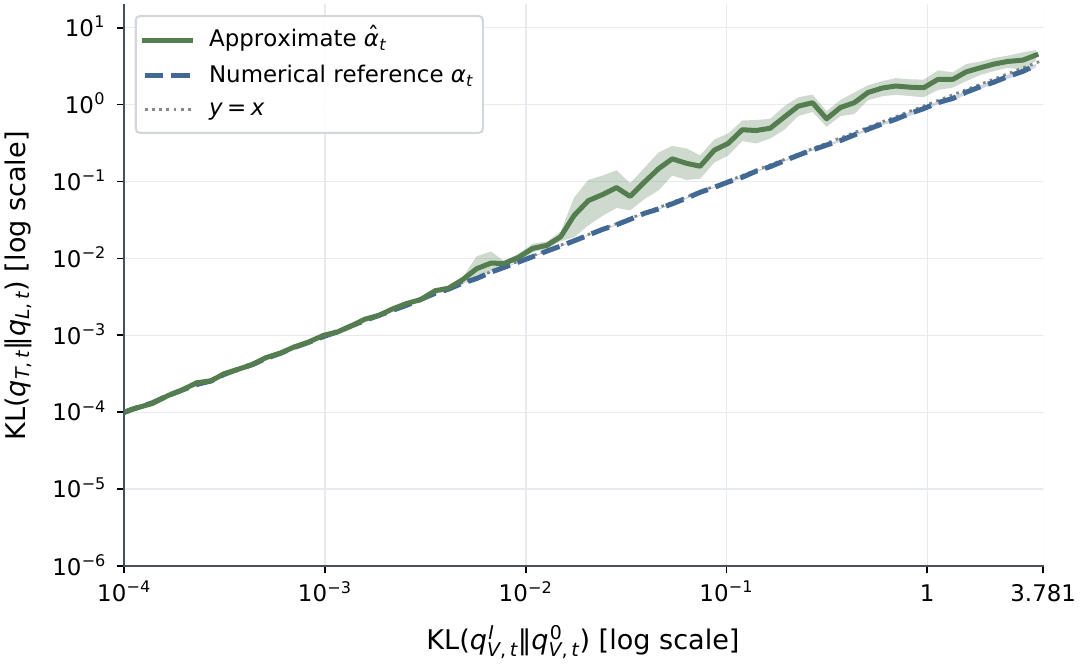}
\par\vspace{-3pt}
{\small (b) Induced teacher KL}
\end{minipage}
\caption{
\textbf{Local approximation error in expert composition.}
Curves show means within source-KL bins at $\rho=1$. The dotted $y=x$ line denotes
source-KL matching. Shaded bands indicate question-bootstrap 95\% confidence intervals.
}
\label{fig:calibration_accuracy}
\end{figure}

To assess the local approximation, we compute calibration coefficients
using the approximation and a numerical KL-matching solver on shared
prefixes for the 4B language and grounding experts.
We use 24,576 shared prefixes from 768 rollouts on 256
MMMU~\citep{yue2024mmmu} questions, using student checkpoints at steps
0, 100, and 200. Experimental details are provided in Appendix~\ref{app:calibration-diagnostics}.
Figure~\ref{fig:calibration_accuracy}(a) shows that the mean approximate
coefficient broadly tracks the numerical solution across source-KL
intervals, with larger deviations at larger source KL.
Figure~\ref{fig:calibration_accuracy}(b) shows that the mean teacher KL
under $\widehat\alpha_t$ closely matches the source KL for small updates
but exceeds it for larger updates. This is consistent with the local
second-order approximation (Equation~\ref{eq:quadratic-calibration}). 
We further examine the effects of approximation error in a practical distillation setting in Appendix~\ref{app:recomposed-training}.

\noindent\textbf{Loss Decomposition and its Interpretation.}
To examine the teacher factors' contributions to the OPD objective,
we express the student's prediction in a factorized form analogous to
the teacher composition.
Let $p_{\theta,t}^{0}(v):=\pi_\theta(v\mid\emptyset,h_t)$
denote the student's language prior and
$r_t^S(v):=\log p_{\theta,t}(v)-\log p_{\theta,t}^{0}(v)$
its visual residual. Thus,
$p_{\theta,t}(v)=p_{\theta,t}^{0}(v)\exp\{r_t^S(v)\}$.
Substituting this factorization and the teacher distribution
in \eqref{eq:family} into the reverse-KL objective gives
\begin{equation}
\KL(p_{\theta,t}\|q_{T,t})
=
\underbrace{
\E_{p_{\theta,t}}
\left[\log\frac{p_{\theta,t}^{0}(v)}{q_{L,t}(v)}\right]
}_{\text{language-prior alignment}}
+
\underbrace{
\E_{p_{\theta,t}}
\left[r_t^S(v)-\widehat\alpha_t r_t(v)\right]
}_{\text{visual-likelihood alignment}}
+
\log Z_t(\widehat\alpha_t).
\label{eq:factorized-loss}
\end{equation}
The first and second terms capture the language-prior and
visual-residual contributions to student--teacher alignment,
respectively, with the grounding expert's residual scaled by
$\widehat\alpha_t$.
The final term normalizes the teacher distribution and is
independent of the student parameters at a fixed prefix.
\eqref{eq:factorized-loss} thus expresses student--teacher
alignment through language-prior and visual-residual contributions
within a single reverse-KL objective.

\section{Experiments}
\label{sec:experiments}

\subsection{Experimental Setup}
\label{sec:experimental-setup}
\noindent\textbf{Models and training data.}
For LEGO-OPD, we use Qwen3-VL-2B-Instruct as the student,
Qwen3-4B-Instruct-2507 as the language expert, and
Qwen3-VL-4B-Instruct as the grounding expert
\citep{yang2025qwen3,bai2025qwen3vl}.
We use the 54,840-example subset of MMFineReason-123K
released by H-OPD as the multimodal training data
\citep{lin2026mmfinereason,yin2026hopd}.
This subset supports comparison with H-OPD and includes
the image descriptions required by its language teacher.

\noindent\textbf{Training.}
All experiments, including ablations, use a common training
implementation, student initialization, and optimization
settings, except for method-specific components.
All student parameters are updated for 200 steps using AdamW
\citep{loshchilov2019adamw} with learning rate $10^{-6}$,
no warmup, and global batch size 128.
Prompt and response lengths are limited to 12,384 tokens.
Rollouts use temperature 1.0 without top-$p$ or top-$k$ truncation.
We follow the top-16 OPD setting of \citet{li2026rethinkingopd}.
Implementation details are in Appendices ~\ref{sec:implementation} and ~\ref{app:training_settings}.

\noindent\textbf{Evaluation.}
We evaluate multimodal reasoning, text-only reasoning, and visual
perception separately, as summarized in Table~\ref{tab:qwen3vl_main}.
The multimodal reasoning benchmarks comprise
MATH-Vision \citep{wang2024mathvision},
MathVerse \citep{zhang2024mathverse},
MMMU-Pro \citep{yue2024mmmupro},
MathVista \citep{lu2024mathvista},
RBench-M \citep{guo2025rbench}, and
OlympiadBench-MM \citep{he2024olympiadbench}.
For text-only evaluation, we use
MATH-500 \citep{lightman2023verify},
MathVerse-Text Only \citep{zhang2024mathverse},
MMLU-Pro \citep{wang2024mmlupro},
RBench-T \citep{guo2025rbench},
OlympiadBench-Text \citep{he2024olympiadbench}.
To assess visual perception separately from reasoning performance,
we evaluate BLINK~\citep{fu2024blink} and
QBench1,2~\citep{wu2023qbench,zhang2024qbenchplus}.
For QBench1,2, we pool the perception-focused Group A questions
from both benchmarks (1,495 + 1,000 questions) and report accuracy
over the combined set.
Except for perception, we average five samples per question.
Category averages are unweighted benchmark means.
Evaluation details are provided in Appendix~\ref{app:evaluation}.

\noindent\textbf{Baselines.}
We compare LEGO-OPD with standard OPD, VA-OPD~\citep{liu2026vaopd},
M-OPD~\citep{ma2026mopd}, and H-OPD~\citep{yin2026hopd}.
For controlled comparison, we match the model scale, optimization settings,
and response budget whenever compatible with each method, while preserving
its defining supervision structure.
M-OPD and H-OPD use both VLM and LLM teachers,
whereas standard OPD and VA-OPD rely only on a VLM teacher.
Since M-OPD is formulated for domain-routed text distillation,
we route our multimodal training data to the VLM teacher
and DeepMath to the LLM teacher.
VA-OPD retains its multi-response sampling scheme, generating four responses
per prompt and therefore observing fewer unique training prompts under the
same response budget.
Further implementation details are provided in
Appendix~\ref{app:baseline_impl}.

\begin{table}[H]
\centering
\caption{
Results on multimodal reasoning, perception, and text reasoning. Bold marks the best distilled-student score.
Evaluation details and SDs are in Appendices~\ref{app:evaluation} and ~\ref{app:seed-results}.
}
\label{tab:qwen3vl_main}

\setlength{\tabcolsep}{3pt}
\renewcommand{\arraystretch}{1.15}
\renewcommand{\theadfont}{\normalsize\bfseries}
\renewcommand{\theadgape}{\Gape[1pt]}
\newcommand{\teachercell}[1]{{\footnotesize\makecell[c]{#1}}}

\resizebox{\linewidth}{!}{%
\begin{tabular}{@{}l*{8}{c}@{}}
\toprule
\noalign{\vskip -1pt}
& \thead{Teacher}
& \thead{Initial\\Student}
& \multicolumn{4}{c}{\textbf{Distillation Baselines}}
& \multicolumn{2}{c}{\textbf{Our Method}} \\
\cmidrule(lr){2-2}
\cmidrule(lr){3-3}
\cmidrule(lr){4-7}
\cmidrule(lr){8-9}
\noalign{\vskip -1pt}

\textbf{Model / Method}
& \thead{Qwen3-VL\\4B-Instruct}
& \thead{Qwen3-VL\\2B-Instruct}
& \thead{Standard\\OPD}
& \thead{VA-OPD}
& \thead{M-OPD}
& \thead{H-OPD}
& \thead{LEGO-OPD\\$\rho=1$}
& \thead{LEGO-OPD\\$\rho=2$} \\
\cmidrule(lr){4-7}
\cmidrule(lr){8-9}
\noalign{\vskip -1pt}

{\footnotesize\textbf{Teacher}}
& \teachercell{--}
& \teachercell{--}
& \teachercell{Qwen3-VL-4B}
& \teachercell{Qwen3-VL-4B}
& \teachercell{Qwen3-VL-4B\\+ Qwen3-4B}
& \teachercell{Qwen3-VL-4B\\+ Qwen3-4B}
& \teachercell{Qwen3-VL-4B\\+ Qwen3-4B}
& \teachercell{Qwen3-VL-4B\\+ Qwen3-4B} \\

\midrule
\multicolumn{9}{c}{\textbf{Multimodal Reasoning}} \\
\midrule
MATH-Vision
& 39.74 & 27.96 & 27.43 & 25.33 & 24.41 & 32.37 & 33.36 & \textbf{34.28} \\
MathVerse (Total)
& 66.70 & 51.51 & 52.70 & 51.66 & 51.28 & 55.41 & 56.96 & \textbf{57.19} \\
MathVista
& 76.26 & 66.56 & 68.60 & 67.36 & 66.42 & 67.70 & 69.62 & \textbf{70.58} \\
MMMU-Pro
& 50.29 & 37.40 & 36.98 & 35.11 & 35.06 & 39.36 & 40.98 & \textbf{41.38} \\
RBench-M
& 32.81 & 19.61 & 16.90 & 16.81 & 15.67 & 26.92 & 29.17 & \textbf{29.20} \\
OlympiadBench-MM
& 63.20 & 37.60 & 39.60 & 37.33 & 36.53 & 39.07 & 44.93 & \textbf{45.07} \\
\midrule
\textbf{Multimodal Avg.}
& 54.83 & 40.11 & 40.37 & 38.93 & 38.23 & 43.47 & 45.84 & \textbf{46.28} \\

\midrule
\multicolumn{9}{c}{\textbf{Perception}} \\
\midrule
BLINK
& 63.49 & 51.34 & 52.13 & \textbf{52.81} & 51.87 & 51.76 & 52.03 & 52.55 \\
QBench1,2
& 76.63 & 69.18 & 69.54 & \textbf{69.70} & 68.86 & 68.98 & 69.34 & 69.62 \\

\midrule
\multicolumn{9}{c}{\textbf{Text Reasoning}} \\
\midrule
MATH-500
& 93.04 & 80.44 & 78.72 & 80.20 & 80.56 & 81.88 & \textbf{82.40} & 82.36 \\
MathVerse (Text Only)
& 66.80 & 57.18 & 57.44 & 56.65 & 56.19 & 59.44 & 59.59 & \textbf{60.23} \\
MMLU-Pro
& 62.90 & 43.59 & 45.52 & 45.09 & 44.98 & 46.94 & 47.15 & \textbf{47.54} \\
RBench-T
& 51.92 & 30.20 & 28.59 & 28.76 & 28.28 & 37.17 & 38.59 & \textbf{40.27} \\
OlympiadBench-Text
& 69.41 & 48.04 & 48.31 & 48.99 & 47.77 & 52.40 & 52.31 & \textbf{53.00} \\
\midrule
\textbf{Text Avg.}
& 68.81 & 51.89 & 51.72 & 51.94 & 51.56 & 55.57 & 56.01 & \textbf{56.68} \\
\bottomrule
\end{tabular}%
}
\end{table}

\subsection{Main Results}

LEGO-OPD at both $\rho=1$ and $\rho=2$ outperforms the evaluated
OPD baselines in multimodal and text-only reasoning averages
(Table~\ref{tab:qwen3vl_main}). Both settings outperform H-OPD
on all six multimodal reasoning benchmarks using the same expert pair.
Increasing $\rho$ from 1 to 2 further improves both reasoning averages
and perception scores. These results support factor composition with
adaptive calibration as an effective way to transfer the experts'
complementary strengths. The performance advantage also holds under
greedy decoding, with the highest multimodal and text-only reasoning
averages among the evaluated 2B students , as reported in
Appendix~\ref{app:greedy-results}.

Table~\ref{tab:calibration_ablation} compares fixed-coefficient
composition with adaptive calibration using the same student,
expert pair, and training settings.
Fixed composition at $\alpha=2$
approaches H-OPD's %
performance
(Tables~\ref{tab:qwen3vl_main} and~\ref{tab:calibration_ablation}).
However, the best fixed coefficient varies across benchmarks,
leaving its choice to empirical tuning.
By contrast, adaptive calibration at $\rho=1$ outperforms all tested
fixed coefficients in both reasoning averages without empirically
tuning a fixed composition coefficient.
Moreover, the gains are larger in multimodal than in text-only reasoning.
This may reflect the need to balance visual grounding with language-based
reasoning. Adaptive calibration can strengthen relevant visual supervision
while preserving language-prior supervision, benefiting 
multimodal reasoning.

With fixed coefficients, increasing $\alpha$ improves perception
but reduces text-only reasoning performance.
Increasing $\alpha$ from 2 to 4 also produces a marked decline in
the multimodal reasoning average.
This pattern suggests that excessive visual emphasis weakens
language-prior supervision.
Related effects have been reported as interference between visual
and language objectives in Decomposed OPD~\citep{yoon2026decomposedopd}
and performance degradation under excessive contrastive shaping
in VCSD~\citep{liang2026vcsd}.
In contrast, adaptive calibration limits changes to the language prior
where the source KL is small.
This helps preserve language-prior supervision even under stronger
visual emphasis.
Across $\rho\in\{1,2,4\}$, the larger budgets improve perception
relative to $\rho=1$ while maintaining or improving both reasoning
averages. This pattern supports adaptive calibration as a way to
increase visual emphasis while preserving reasoning performance.
Appendix~\ref{app:zero-budget-controls} additionally describes $\rho=0$ setting.

\setlength{\textfloatsep}{18pt plus 1pt minus 1pt}

\begingroup
\setlength{\intextsep}{6pt}
\definecolor{fixedgreen}{HTML}{547D50}
\definecolor{adaptiveblue}{HTML}{416994}
\newcommand{\fixedbest}[1]{\textcolor{fixedgreen}{\textbf{#1}}}
\newcommand{\adaptivebest}[1]{\textcolor{adaptiveblue}{\textbf{#1}}}

\begin{table}[!t]
\centering
\setlength{\abovecaptionskip}{0pt}
\setlength{\belowcaptionskip}{2pt}
\caption{\textbf{Adaptive calibration and the KL-budget scale.}
Generalized Bayesian composition with fixed coefficients
or adaptive calibration.
Best scores per group are
\fixedbest{green} (fixed) and \adaptivebest{blue} (adaptive).
Reasoning standard deviations are in Appendix~\ref{app:calibration-results}.}
\label{tab:calibration_ablation}
\footnotesize
\setlength{\tabcolsep}{3pt}
\renewcommand{\arraystretch}{0.95}
\setlength{\aboverulesep}{1pt}
\setlength{\belowrulesep}{1pt}
\resizebox{0.95\linewidth}{!}{%
\begin{tabular*}{\textwidth}{@{\extracolsep{\fill}}l*{6}{c}@{}}
\toprule
& \multicolumn{3}{c}{\textbf{Fixed coefficient}}
& \multicolumn{3}{c}{\textbf{Adaptive}} \\
\cmidrule(lr){2-4}\cmidrule(lr){5-7}
\textbf{Benchmark} & $\alpha=1$ & $\alpha=2$ & $\alpha=4$
& $\rho=1$ & $\rho=2$ & $\rho=4$ \\
\midrule
\multicolumn{7}{c}{\textbf{Multimodal Reasoning}} \\
\midrule
MATH-Vision & \fixedbest{31.91} & 31.25 & 30.00 & 33.36 & \adaptivebest{34.28} & 34.01 \\
MathVerse (Total) & 53.46 & \fixedbest{54.29} & 53.41 & 56.96 & 57.19 & \adaptivebest{57.51} \\
MathVista & 64.58 & 68.98 & \fixedbest{69.36} & 69.62 & \adaptivebest{70.58} & 70.50 \\
MMMU-Pro & \fixedbest{39.48} & 38.87 & 37.45 & 40.98 & \adaptivebest{41.38} & 41.36 \\
RBench-M & \fixedbest{26.89} & 24.03 & 17.62 & 29.17 & \adaptivebest{29.20} & 28.39 \\
OlympiadBench-MM & 38.27 & \fixedbest{41.87} & 41.73 & 44.93 & 45.07 & \adaptivebest{45.33} \\
\midrule
\textbf{Multimodal Avg.} & 42.43 & \fixedbest{43.22} & 41.60 & 45.84 & \adaptivebest{46.28} & 46.18 \\
\midrule
\multicolumn{7}{c}{\textbf{Perception}} \\
\midrule
BLINK & 51.18 & 51.29 & \fixedbest{51.92} & 52.03 & \adaptivebest{52.55} & 52.50 \\
QBench1,2 & 68.46 & 68.86 & \fixedbest{69.66} & 69.34 & 69.62 & \adaptivebest{69.70} \\
\midrule
\multicolumn{7}{c}{\textbf{Text Reasoning}} \\
\midrule
MATH-500 & \fixedbest{82.24} & 82.08 & 81.52 & 82.40 & 82.36 & \adaptivebest{82.52} \\
MathVerse (Text Only) & \fixedbest{59.06} & 58.63 & 57.94 & 59.59 & 60.23 & \adaptivebest{60.46} \\
MMLU-Pro & \fixedbest{47.40} & 47.16 & 47.10 & 47.15 & \adaptivebest{47.54} & 47.38 \\
RBench-T & \fixedbest{38.21} & 35.52 & 32.85 & 38.59 & \adaptivebest{40.27} & 38.08 \\
OlympiadBench-Text & \fixedbest{52.20} & 51.72 & 51.19 & 52.31 & \adaptivebest{53.00} & 51.78 \\
\midrule
\textbf{Text Avg.} & \fixedbest{55.82} & 55.02 & 54.12 & 56.01 & \adaptivebest{56.68} & 56.04 \\
\bottomrule
\end{tabular*}%
}
\end{table}
\endgroup

\subsection{Analysis of Language and Grounding Experts}
\label{sec:expert-contributions}
Table~\ref{tab:expert_contributions} compares the full 4B--4B composition
with variants that replace one expert's factor with that of the frozen
initial student. To focus on reasoning and grounding, the subset includes
multimodal mathematics and challenging reasoning benchmarks alongside
their available text-only counterparts.
Table~\ref{tab:expert_combinations_full} reports the full reasoning benchmark
results and extends the comparison to additional expert combinations.
The experiments use Qwen3-1.7B and Qwen3-4B-Instruct-2507 as language
experts and Qwen3-VL-Instruct models of sizes 2B, 4B, and 8B as grounding
experts. One expert is varied at a time from the default 4B--4B pairing.

\begingroup
\begin{table}[!ht]
\centering
\setlength{\abovecaptionskip}{0pt}
\setlength{\belowcaptionskip}{2pt}
\caption{\textbf{Contributions of the language and grounding experts.}
Student (2B) is the frozen initial Qwen3-VL-2B.
All runs use $\rho=1$.
Parentheses show pp changes from the full composition.}
\label{tab:expert_contributions}
\setlength{\tabcolsep}{5pt}
\renewcommand{\arraystretch}{1.15}
\renewcommand{\theadfont}{\normalsize\bfseries}
\renewcommand{\theadgape}{\Gape[1pt]}
\newcommand{\alignedscore}[2]{%
  \makebox[5.2em][c]{#1\rlap{\hspace{0.15em}{\scriptsize\color{black!60}#2}}}%
}
\resizebox{0.95\textwidth}{!}{%
\begin{tabular}{@{}*{7}{c}@{\hspace{1.6em}}}
\toprule
\multicolumn{2}{c}{\textbf{Teacher composition}}
& \multicolumn{3}{c}{\textbf{Multimodal Reasoning}}
& \multicolumn{2}{c}{\textbf{Text Reasoning}} \\
\cmidrule(lr){1-2}
\cmidrule(lr){3-5}
\cmidrule(lr){6-7}
\thead{Language\\prior}
& \thead{Visual\\residual}
& \thead{MATH-Vision}
& \thead{MathVerse\\(Total)}
& \thead{Olympiad\\Bench-MM}
& \thead{MathVerse\\(Text Only)}
& \thead{Olympiad\\Bench-Text} \\
\midrule
Qwen3-4B & Qwen3-VL-4B
& \alignedscore{\textbf{33.36}}{}
& \alignedscore{\textbf{56.96}}{}
& \alignedscore{\textbf{44.93}}{}
& \alignedscore{59.59}{}
& \alignedscore{52.31}{} \\
\midrule
Qwen3-4B & Student (2B)
& \alignedscore{33.03}{($-0.33$)}
& \alignedscore{54.17}{($-2.79$)}
& \alignedscore{43.07}{($-1.86$)}
& \alignedscore{\textbf{59.92}}{($+0.33$)}
& \alignedscore{\textbf{52.67}}{($+0.36$)} \\
Student (2B) & Qwen3-VL-4B
& \alignedscore{26.58}{($-6.78$)}
& \alignedscore{52.69}{($-4.27$)}
& \alignedscore{37.20}{($-7.73$)}
& \alignedscore{56.98}{($-2.61$)}
& \alignedscore{48.07}{($-4.24$)} \\
\bottomrule
\end{tabular}%
}
\end{table}
\endgroup

A key observation in Table~\ref{tab:expert_contributions} is that the two
experts' benefits combine in the full composition.
The language prior supports reasoning across modalities, while the visual
residual provides additional multimodal gains with little change in
text-only performance. This pattern supports the joint transfer of
language reasoning and visual grounding to the student.

\begingroup
\begin{table}[!t]
\centering
\setlength{\abovecaptionskip}{0pt}
\setlength{\belowcaptionskip}{2pt}
\caption{\textbf{Effect of expert choice on student reasoning.}
One expert is varied from the shaded Base.
All runs use the same 2B student, 200 steps, and $\rho=1$.
Student (2B) denotes the frozen initial student.
When used as the language expert, it receives no image.
Parentheses show average changes from Base in pp (blue increases, red decreases).}
\label{tab:expert_combinations_full}
\definecolor{expertgain}{HTML}{285D9C}
\definecolor{expertloss}{HTML}{B04B4B}
\newcommand{\expertavg}[3]{#1\,{\fontsize{6}{7}\selectfont\textcolor{#2}{($#3$)}}}
\footnotesize
\setlength{\tabcolsep}{3pt}
\renewcommand{\arraystretch}{1.18}
\setlength{\aboverulesep}{0.4pt}
\setlength{\belowrulesep}{0.4pt}
\setlength{\cmidrulesep}{0pt}
\renewcommand{\theadfont}{\footnotesize\bfseries}
\renewcommand{\theadgape}{\Gape[1pt]}
\resizebox{0.95\linewidth}{!}{%
\begin{tabular*}{\textwidth}{@{\extracolsep{\fill}}l*{2}{c}>{\columncolor{black!6}}c*{2}{c}@{}}
\toprule
\textbf{Teacher composition} & \multicolumn{2}{c}{\textbf{Language expert swap}} & \textbf{Base} & \multicolumn{2}{c}{\textbf{Grounding expert swap}} \\
\cmidrule(lr){2-3}\cmidrule(lr){5-6}
Language expert & LLM 1.7B & Student (2B) & \textbf{LLM 4B} & LLM 4B & LLM 4B \\
Grounding expert & VLM 4B & VLM 4B & \textbf{VLM 4B} & Student (2B) & VLM 8B \\
\midrule
\multicolumn{6}{c}{\textbf{Multimodal Reasoning}} \\
\midrule
MATH-Vision & 28.36 & 26.58 & 33.36 & 33.03 & 34.34 \\
MathVerse (Total) & 43.89 & 52.69 & 56.96 & 54.17 & 57.81 \\
MathVista & 64.96 & 68.20 & 69.62 & 69.42 & 70.24 \\
MMMU-Pro & 35.11 & 36.36 & 40.98 & 41.93 & 41.70 \\
RBench-M & 22.71 & 17.53 & 29.17 & 28.48 & 28.90 \\
OlympiadBench-MM & 36.00 & 37.20 & 44.93 & 43.07 & 45.20 \\
\midrule
\textbf{Multimodal Avg.} & \expertavg{38.51}{expertloss}{-7.33} & \expertavg{39.76}{expertloss}{-6.08} & 45.84 & \expertavg{45.02}{expertloss}{-0.82} & \expertavg{46.37}{expertgain}{+0.53} \\
\midrule
\multicolumn{6}{c}{\textbf{Text Reasoning}} \\
\midrule
MATH-500 & 81.52 & 79.60 & 82.40 & 82.36 & 84.08 \\
MathVerse (Text Only) & 55.03 & 56.98 & 59.59 & 59.92 & 59.59 \\
MMLU-Pro & 43.14 & 44.91 & 47.15 & 47.51 & 46.92 \\
RBench-T & 33.05 & 29.87 & 38.59 & 38.46 & 37.61 \\
OlympiadBench-Text & 47.57 & 48.07 & 52.31 & 52.67 & 53.06 \\
\midrule
\textbf{Text Avg.} & \expertavg{52.06}{expertloss}{-3.95} & \expertavg{51.89}{expertloss}{-4.12} & 56.01 & \expertavg{56.18}{expertgain}{+0.17} & \expertavg{56.25}{expertgain}{+0.24} \\
\bottomrule
\end{tabular*}%
}
\end{table}
\endgroup

The expanded results in Table~\ref{tab:expert_combinations_full} clarify
the effects of changing each expert.
With the 4B grounding expert fixed, replacing the base 4B LLM with the
1.7B LLM lowers both text-only and multimodal reasoning scores.
Conversely, with the 4B language expert fixed, increasing the grounding
model from 2B to 4B to 8B improves the multimodal reasoning average,
while text-only averages remain nearly unchanged.
Settings and standard deviations are in Appendix~\ref{app:additional-experts}.

\section{Conclusion}

We presented \textbf{LEGO-OPD}, which composes a language prior and a visual grounding signal into a unified teacher for multimodal on-policy distillation.
A generalized Bayesian formulation enables factorized teacher composition, while source-referenced calibration adaptively controls the visual contribution at each student-generated prefix.
A local KL approximation yields a closed-form calibration coefficient, with its approximation error assessed against a numerical KL-matching reference.
Across Qwen3 models, LEGO-OPD achieves the best average performance among the evaluated single- and multi-teacher OPD baselines on both multimodal and text-only reasoning.
Ablations show that the language prior and visual residual provide complementary benefits, and that adaptive calibration outperforms fixed weighting in both reasoning averages while allowing stronger visual emphasis without degrading reasoning.
Overall, LEGO-OPD provides an effective framework for balancing language reasoning and visual grounding in multimodal distillation. Limitations are discussed in Appendix~\ref{sec:limitations}.

\section*{Reproducibility statement}
Section~\ref{sec:experimental-setup} describes the models, training data,
optimization settings, and evaluation benchmarks used in our experiments.
Appendix~\ref{app:proofs} provides the theoretical analysis and proofs,
and Appendix~\ref{sec:implementation} specifies the practical teacher
construction and student update of LEGO-OPD.
Appendix~\ref{app:experimental-details} details the training settings,
evaluation protocols, prompt templates, and baseline implementations,
and Appendix~\ref{app:practical-validation} gives the configurations
for the calibration ablations and diagnostics.
The code is available at \url{https://anonymous.4open.science/r/LEGO_OPD-51C1}.

\section*{AI use statement}
We used generative AI tools as interactive research assistants throughout this work. Our use was mainly iterative. We used them to explore alternative formulations, examine our reasoning, identify ambiguities or potential weaknesses, and obtain feedback as the work developed. They also helped broaden our literature review, refine translation and English scientific writing, and accelerate implementation and debugging.

Generative AI tools also supported discussion of theoretical formulations and mathematical derivations, experimental design, result interpretation, and code development. These tools informed the research process, while the authors set the research direction, made the methodological and experimental decisions, and determined the final interpretations. AI-assisted suggestions were treated as material for further examination rather than authoritative outputs. Mathematical arguments were checked by the authors, AI-assisted implementations were manually reviewed and tested against the intended methods and experimental settings, and relevant claims and citations were verified against their original sources. The authors repeatedly revised the manuscript using AI tools for discussion, critique, translation, and editing while continuing to develop and examine the underlying arguments themselves. Generative AI was also used as a documented part of the evaluation pipeline for semantic answer interpretation, as described in Appendix ~\ref{app:evaluation}.

\bibliographystyle{iclr2027_conference}
\bibliography{fusion_opd_references_checked}

\begin{thebibliography}{32}
\providecommand{\natexlab}[1]{#1}
\providecommand{\url}[1]{\texttt{#1}}
\expandafter\ifx\csname urlstyle\endcsname\relax
  \providecommand{\doi}[1]{doi: #1}\else
  \providecommand{\doi}{doi: \begingroup \urlstyle{rm}\Url}\fi

\bibitem[Agarwal et~al.(2024)Agarwal, Vieillard, Zhou, Stanczyk, Ramos~Garea,
  Geist, and Bachem]{agarwal2024gkd}
Rishabh Agarwal, Nino Vieillard, Yongchao Zhou, Piotr Stanczyk, Sabela
  Ramos~Garea, Matthieu Geist, and Olivier Bachem.
\newblock On-policy distillation of language models: Learning from
  self-generated mistakes.
\newblock In B.~Kim, Y.~Yue, S.~Chaudhuri, K.~Fragkiadaki, M.~Khan, and Y.~Sun
  (eds.), \emph{International Conference on Learning Representations}, volume
  2024, pp.\  21246--21263, 2024.
\newblock URL
  \url{https://proceedings.iclr.cc/paper_files/paper/2024/file/5be69a584901a26c521c2b51e40a4c20-Paper-Conference.pdf}.

\bibitem[Bai et~al.(2025)Bai, Cai, Chen, Chen, Chen, Cheng, Deng, Ding, Gao,
  Ge, Ge, Guo, Huang, Huang, Huang, Hui, Jiang, Li, Li, Li, Li, Lin, Lin, Liu,
  Liu, Liu, Liu, Liu, Liu, Lu, Luo, Lv, Men, Meng, Ren, Ren, Song, Sun, Tang,
  Tu, Wan, Wang, Wang, Wang, Wang, Xie, Xu, Xu, Xu, Yang, Yang, Yang, Yang, Yu,
  Zhang, Zhang, Zhang, Zheng, Zhong, Zhou, Zhou, Zhou, Zhu, and
  Zhu]{bai2025qwen3vl}
Shuai Bai, Yuxuan Cai, Ruizhe Chen, Keqin Chen, Xionghui Chen, Zesen Cheng,
  Lianghao Deng, Wei Ding, Chang Gao, Chunjiang Ge, Wenbin Ge, Zhifang Guo,
  Qidong Huang, Jie Huang, Fei Huang, Binyuan Hui, Shutong Jiang, Zhaohai Li,
  Mingsheng Li, Mei Li, Kaixin Li, Zicheng Lin, Junyang Lin, Xuejing Liu,
  Jiawei Liu, Chenglong Liu, Yang Liu, Dayiheng Liu, Shixuan Liu, Dunjie Lu,
  Ruilin Luo, Chenxu Lv, Rui Men, Lingchen Meng, Xuancheng Ren, Xingzhang Ren,
  Sibo Song, Yuchong Sun, Jun Tang, Jianhong Tu, Jianqiang Wan, Peng Wang,
  Pengfei Wang, Qiuyue Wang, Yuxuan Wang, Tianbao Xie, Yiheng Xu, Haiyang Xu,
  Jin Xu, Zhibo Yang, Mingkun Yang, Jianxin Yang, An~Yang, Bowen Yu, Fei Zhang,
  Hang Zhang, Xi~Zhang, Bo~Zheng, Humen Zhong, Jingren Zhou, Fan Zhou, Jing
  Zhou, Yuanzhi Zhu, and Ke~Zhu.
\newblock {Qwen3-VL} technical report, 2025.
\newblock URL \url{https://arxiv.org/abs/2511.21631}.

\bibitem[Bissiri et~al.(2016)Bissiri, Holmes, and Walker]{bissiri2016belief}
P.~G. Bissiri, C.~C. Holmes, and S.~G. Walker.
\newblock A general framework for updating belief distributions.
\newblock \emph{Journal of the Royal Statistical Society Series B: Statistical
  Methodology}, 78\penalty0 (5):\penalty0 1103--1130, 11 2016.
\newblock ISSN 1369-7412.
\newblock \doi{10.1111/rssb.12158}.
\newblock URL \url{https://doi.org/10.1111/rssb.12158}.

\bibitem[Fu et~al.(2025)Fu, Hu, Li, Feng, Wang, Lin, Roth, Smith, Ma, and
  Krishna]{fu2024blink}
Xingyu Fu, Yushi Hu, Bangzheng Li, Yu~Feng, Haoyu Wang, Xudong Lin, Dan Roth,
  Noah~A. Smith, Wei-Chiu Ma, and Ranjay Krishna.
\newblock {BLINK}: Multimodal large language models can see but not perceive.
\newblock In Ale{\v{s}} Leonardis, Elisa Ricci, Stefan Roth, Olga Russakovsky,
  Torsten Sattler, and G{\"u}l Varol (eds.), \emph{Computer Vision -- ECCV
  2024}, volume 15081 of \emph{Lecture Notes in Computer Science}, pp.\
  148--166, Cham, 2025. Springer Nature Switzerland.
\newblock ISBN 978-3-031-73337-6.
\newblock \doi{10.1007/978-3-031-73337-6_9}.
\newblock URL \url{https://doi.org/10.1007/978-3-031-73337-6_9}.

\bibitem[Guo et~al.(2025)Guo, Xu, Zhang, Song, Peng, Deng, Dong, Nakayama,
  Geng, Wang, Ni, Yang, Rao, Peng, Hu, Wetzstein, and Hu]{guo2025rbench}
Meng-Hao Guo, Jiajun Xu, Yi~Zhang, Jiaxi Song, Haoyang Peng, Yi-Xuan Deng,
  Xinzhi Dong, Kiyohiro Nakayama, Zhengyang Geng, Chen Wang, Bolin Ni, Guo-Wei
  Yang, Yongming Rao, Houwen Peng, Han Hu, Gordon Wetzstein, and Shi-Min Hu.
\newblock {RB}ench: Graduate-level multi-disciplinary benchmarks for {LLM} \&
  {MLLM} complex reasoning evaluation.
\newblock In Aarti Singh, Maryam Fazel, Daniel Hsu, Simon Lacoste-Julien, Felix
  Berkenkamp, Tegan Maharaj, Kiri Wagstaff, and Jerry Zhu (eds.),
  \emph{Proceedings of the 42nd International Conference on Machine Learning},
  volume 267 of \emph{Proceedings of Machine Learning Research}, pp.\
  21184--21201. PMLR, 2025.
\newblock URL \url{https://proceedings.mlr.press/v267/guo25r.html}.

\bibitem[He et~al.(2024)He, Luo, Bai, Hu, Thai, Shen, Hu, Han, Huang, Zhang,
  Liu, Qi, Liu, and Sun]{he2024olympiadbench}
Chaoqun He, Renjie Luo, Yuzhuo Bai, Shengding Hu, Zhen Thai, Junhao Shen, Jinyi
  Hu, Xu~Han, Yujie Huang, Yuxiang Zhang, Jie Liu, Lei Qi, Zhiyuan Liu, and
  Maosong Sun.
\newblock {O}lympiad{B}ench: A challenging benchmark for promoting {AGI} with
  olympiad-level bilingual multimodal scientific problems.
\newblock In Lun-Wei Ku, Andre Martins, and Vivek Srikumar (eds.),
  \emph{Proceedings of the 62nd Annual Meeting of the Association for
  Computational Linguistics (Volume 1: Long Papers)}, pp.\  3828--3850,
  Bangkok, Thailand, August 2024. Association for Computational Linguistics.
\newblock \doi{10.18653/v1/2024.acl-long.211}.
\newblock URL \url{https://aclanthology.org/2024.acl-long.211/}.

\bibitem[He et~al.(2026)He, Liang, Xu, Liu, Chen, Wang, Song, Yu, Liang, Wang,
  Zhang, Wang, Tu, Mi, and Yu]{he2025deepmath103k}
Zhiwei He, Tian Liang, Jiahao Xu, Qiuzhi Liu, Xingyu Chen, Yue Wang, Linfeng
  Song, Dian Yu, Zhenwen Liang, Wenxuan Wang, Zhuosheng Zhang, Rui Wang,
  Zhaopeng Tu, Haitao Mi, and Dong Yu.
\newblock {DeepMath-103K}: A large-scale, challenging, decontaminated, and
  verifiable mathematical dataset for advancing reasoning.
\newblock In \emph{The Fourteenth International Conference on Learning
  Representations}, 2026.
\newblock URL \url{https://openreview.net/forum?id=kHB5Te5IWm}.

\bibitem[Kim et~al.(2026)Kim, Chang, Hwang, Kim, and Ye]{kim2026unir}
Jaemin Kim, Hangeol Chang, Hyunmin Hwang, Choonghan Kim, and Jong~Chul Ye.
\newblock Universal reasoner: A single, composable plug-and-play reasoner for
  frozen {LLM}s.
\newblock In \emph{Forty-third International Conference on Machine Learning},
  2026.
\newblock URL \url{https://openreview.net/forum?id=JheeAwquCb}.

\bibitem[{Kwai Keye Team} et~al.(2026){Kwai Keye Team}, Wen, Liu, Song, Rao,
  Zhang, Li, Fan, Ju, Chen, Chen, Yuan, Yang, Jiang, Gai, Zhou, Nie, Na, Zhang,
  Gao, Zheng, Chen, Yang, Gao, Yang, Liu, Diao, Zhang, Su, Chen, Hong, Lu,
  Long, Yang, Li, Fan, Xia, Chen, Lai, Yi, Jia, Liang, Xu, Ma, Tian, Han, Han,
  Li, Wang, Jia, Chen, Shi, and Zhang]{kwaikeyeteam2026keyevl2}
{Kwai Keye Team}, Bin Wen, Changyi Liu, Chengru Song, Chongling Rao, Guowang
  Zhang, Han Li, Haonan Fan, Hengrui Ju, Jiankang Chen, Jiapeng Chen, Jiawei
  Yuan, Kaixuan Yang, Kaiyu Jiang, Kun Gai, Lingzhi Zhou, Na~Nie, Sen Na,
  Tianke Zhang, Tingting Gao, Xuanyu Zheng, Yulong Chen, Fan Yang, Haixuan Gao,
  Lele Yang, Mingqiao Liu, Muxi Diao, Qi~Zhang, Qile Su, Wei Chen, Wentao Hong,
  Xingyu Lu, Yancheng Long, Yankai Yang, Yingxin Li, Yiyang Fan, Yu~Xia, Yuzhe
  Chen, Ziliang Lai, Chuan Yi, Haonan Jia, Tianming Liang, Weixin Xu, Xiaoxiao
  Ma, Yang Tian, Yufei Han, Feng Han, Hang Li, Jing Wang, Jinghui Jia, Junmin
  Chen, Junyu Shi, and Ruilin Zhang.
\newblock {Kwai Keye-VL-2.0} technical report, 2026.
\newblock URL \url{https://arxiv.org/abs/2606.10651}.

\bibitem[Leng et~al.(2024)Leng, Zhang, Chen, Li, Lu, Miao, and
  Bing]{leng2023vcd}
Sicong Leng, Hang Zhang, Guanzheng Chen, Xin Li, Shijian Lu, Chunyan Miao, and
  Lidong Bing.
\newblock Mitigating object hallucinations in large vision-language models
  through visual contrastive decoding.
\newblock In \emph{Proceedings of the IEEE/CVF Conference on Computer Vision
  and Pattern Recognition (CVPR)}, pp.\  13872--13882, June 2024.
\newblock URL
  \url{https://openaccess.thecvf.com/content/CVPR2024/html/Leng_Mitigating_Object_Hallucinations_in_Large_Vision-Language_Models_through_Visual_Contrastive_CVPR_2024_paper.html}.

\bibitem[Li et~al.(2026)Li, Zuo, He, Zhang, Xiao, Qian, Yu, ang Gao, Yang, Liu,
  and Ding]{li2026rethinkingopd}
Yaxuan Li, Yuxin Zuo, Bingxiang He, Jinqian Zhang, Chaojun Xiao, Cheng Qian,
  Tianyu Yu, Huan ang Gao, Wenkai Yang, Zhiyuan Liu, and Ning Ding.
\newblock Rethinking on-policy distillation of large language models:
  Phenomenology, mechanism, and recipe, 2026.
\newblock URL \url{https://arxiv.org/abs/2604.13016}.

\bibitem[Liang et~al.(2026)Liang, Tian, Li, Jia, Huang, Zhou, and
  Fu]{liang2026vcsd}
Yijun Liang, Yunjie Tian, Yijiang Li, Yuqi Jia, Furong Huang, Tianyi Zhou, and
  Di~Fu.
\newblock Visual contrastive self-distillation, 2026.
\newblock URL \url{https://arxiv.org/abs/2607.21556}.

\bibitem[Lightman et~al.(2024)Lightman, Kosaraju, Burda, Edwards, Baker, Lee,
  Leike, Schulman, Sutskever, and Cobbe]{lightman2023verify}
Hunter Lightman, Vineet Kosaraju, Yuri Burda, Harrison Edwards, Bowen Baker,
  Teddy Lee, Jan Leike, John Schulman, Ilya Sutskever, and Karl Cobbe.
\newblock Let's verify step by step.
\newblock In B.~Kim, Y.~Yue, S.~Chaudhuri, K.~Fragkiadaki, M.~Khan, and Y.~Sun
  (eds.), \emph{International Conference on Learning Representations}, volume
  2024, pp.\  39578--39601, 2024.
\newblock URL
  \url{https://proceedings.iclr.cc/paper_files/paper/2024/file/aca97732e30bcf1303bc22ac3924fd16-Paper-Conference.pdf}.

\bibitem[Lin et~al.(2026)Lin, Liu, Zhu, Qin, Lin, Shang, He, Zhang, and
  Wu]{lin2026mmfinereason}
Honglin Lin, Zheng Liu, Yun Zhu, Chonghan Qin, Juekai Lin, Xiaoran Shang,
  Conghui He, Wentao Zhang, and Lijun Wu.
\newblock {MMFineReason}: Closing the multimodal reasoning gap via open
  data-centric methods, 2026.
\newblock URL \url{https://arxiv.org/abs/2601.21821}.

\bibitem[Liu et~al.(2021)Liu, Sap, Lu, Swayamdipta, Bhagavatula, Smith, and
  Choi]{liu2021dexperts}
Alisa Liu, Maarten Sap, Ximing Lu, Swabha Swayamdipta, Chandra Bhagavatula,
  Noah~A. Smith, and Yejin Choi.
\newblock {DE}xperts: Decoding-time controlled text generation with experts and
  anti-experts.
\newblock In Chengqing Zong, Fei Xia, Wenjie Li, and Roberto Navigli (eds.),
  \emph{Proceedings of the 59th Annual Meeting of the Association for
  Computational Linguistics and the 11th International Joint Conference on
  Natural Language Processing (Volume 1: Long Papers)}, pp.\  6691--6706,
  Online, August 2021. Association for Computational Linguistics.
\newblock \doi{10.18653/v1/2021.acl-long.522}.
\newblock URL \url{https://aclanthology.org/2021.acl-long.522/}.

\bibitem[Liu et~al.(2024)Liu, Han, Wang, Tsvetkov, Choi, and
  Smith]{liu2024proxytuning}
Alisa Liu, Xiaochuang Han, Yizhong Wang, Yulia Tsvetkov, Yejin Choi, and
  Noah~A. Smith.
\newblock Tuning language models by proxy.
\newblock In \emph{First Conference on Language Modeling}, 2024.
\newblock URL \url{https://openreview.net/forum?id=dribhnhm1i}.

\bibitem[Liu et~al.(2026)Liu, Lv, Li, Zhu, Wang, Zhang, Chen, Li, Li, Gao, and
  Wu]{liu2026vaopd}
Ruiqi Liu, Xiaolei Lv, Gengsheng Li, Ximo Zhu, Zhiheng Wang, Zhengbo Zhang,
  Junkai Chen, Zhiheng Li, Bo~Li, Jun Gao, and Shu Wu.
\newblock Visual-advantage on-policy distillation for vision-language models,
  2026.
\newblock URL \url{https://arxiv.org/abs/2605.21924}.

\bibitem[Loshchilov \& Hutter(2019)Loshchilov and Hutter]{loshchilov2019adamw}
Ilya Loshchilov and Frank Hutter.
\newblock Decoupled weight decay regularization.
\newblock In \emph{International Conference on Learning Representations}, 2019.
\newblock URL \url{https://openreview.net/forum?id=Bkg6RiCqY7}.

\bibitem[Lu et~al.(2024)Lu, Bansal, Xia, Liu, Li, Hajishirzi, Cheng, Chang,
  Galley, and Gao]{lu2024mathvista}
Pan Lu, Hritik Bansal, Tony Xia, Jiacheng Liu, Chunyuan Li, Hannaneh
  Hajishirzi, Hao Cheng, Kai-Wei Chang, Michel Galley, and Jianfeng Gao.
\newblock {MathVista}: Evaluating mathematical reasoning of foundation models
  in visual contexts.
\newblock In B.~Kim, Y.~Yue, S.~Chaudhuri, K.~Fragkiadaki, M.~Khan, and Y.~Sun
  (eds.), \emph{International Conference on Learning Representations}, volume
  2024, pp.\  23439--23554, 2024.
\newblock URL
  \url{https://proceedings.iclr.cc/paper_files/paper/2024/file/663bce02a0050c4a11f1eb8a7f1429d3-Paper-Conference.pdf}.

\bibitem[Ma et~al.(2026)Ma, Wei, Zhao, Zhang, Xiao, Li, Yang, Gao, Wang, Li,
  Dong, Sui, and Luo]{ma2026mopd}
Wenhan Ma, Jianyu Wei, Liang Zhao, Hailin Zhang, Bangjun Xiao, Lei Li, Qibin
  Yang, Bofei Gao, Yudong Wang, Rang Li, Jinhao Dong, Zhifang Sui, and Fuli
  Luo.
\newblock {MOPD}: Multi-teacher on-policy distillation for capability
  integration in {LLM} post-training, 2026.
\newblock URL \url{https://arxiv.org/abs/2606.30406}.

\bibitem[{Qwen Team}(2025)]{qwen2025qwen3_4b_instruct2507}
{Qwen Team}.
\newblock {Qwen3-4B-Instruct-2507}: Model card.
\newblock Hugging Face model card, 2025.
\newblock URL \url{https://huggingface.co/Qwen/Qwen3-4B-Instruct-2507}.

\bibitem[Wang et~al.(2024{\natexlab{a}})Wang, Pan, Shi, Lu, Ren, Zhou, Zhan,
  and Li]{wang2024mathvision}
Ke~Wang, Junting Pan, Weikang Shi, Zimu Lu, Houxing Ren, Aojun Zhou, Mingjie
  Zhan, and Hongsheng Li.
\newblock Measuring multimodal mathematical reasoning with {MATH-Vision}
  dataset.
\newblock In A.~Globerson, L.~Mackey, D.~Belgrave, A.~Fan, U.~Paquet,
  J.~Tomczak, and C.~Zhang (eds.), \emph{Advances in Neural Information
  Processing Systems}, volume~37, pp.\  95095--95169. Curran Associates, Inc.,
  2024{\natexlab{a}}.
\newblock \doi{10.52202/079017-3014}.
\newblock URL
  \url{https://proceedings.neurips.cc/paper_files/paper/2024/file/ad0edc7d5fa1a783f063646968b7315b-Paper-Datasets_and_Benchmarks_Track.pdf}.

\bibitem[Wang et~al.(2026)Wang, Liu, Liu, and Han]{wang2026vigos}
Sihan Wang, Xiyao Liu, Lianqing Liu, and Zhi Han.
\newblock Seeing before reasoning: Decoupling perception and reasoning for
  shortcut-resilient multimodal on-policy self-distillation, 2026.
\newblock URL \url{https://arxiv.org/abs/2606.19120}.

\bibitem[Wang et~al.(2024{\natexlab{b}})Wang, Ma, Zhang, Ni, Chandra, Guo, Ren,
  Arulraj, He, Jiang, Li, Ku, Wang, Zhuang, Fan, Yue, and
  Chen]{wang2024mmlupro}
Yubo Wang, Xueguang Ma, Ge~Zhang, Yuansheng Ni, Abhranil Chandra, Shiguang Guo,
  Weiming Ren, Aaran Arulraj, Xuan He, Ziyan Jiang, Tianle Li, Max Ku, Kai
  Wang, Alex Zhuang, Rongqi Fan, Xiang Yue, and Wenhu Chen.
\newblock {MMLU-Pro}: A more robust and challenging multi-task language
  understanding benchmark.
\newblock In A.~Globerson, L.~Mackey, D.~Belgrave, A.~Fan, U.~Paquet,
  J.~Tomczak, and C.~Zhang (eds.), \emph{Advances in Neural Information
  Processing Systems}, volume~37, pp.\  95266--95290. Curran Associates, Inc.,
  2024{\natexlab{b}}.
\newblock \doi{10.52202/079017-3018}.
\newblock URL
  \url{https://proceedings.neurips.cc/paper_files/paper/2024/file/ad236edc564f3e3156e1b2feafb99a24-Paper-Datasets_and_Benchmarks_Track.pdf}.

\bibitem[Wu et~al.(2024)Wu, Zhang, Zhang, Chen, Liao, Wang, Li, Sun, Yan, Zhai,
  and Lin]{wu2023qbench}
Haoning Wu, Zicheng Zhang, Erli Zhang, Chaofeng Chen, Liang Liao, Annan Wang,
  Chunyi Li, Wenxiu Sun, Qiong Yan, Guangtao Zhai, and Weisi Lin.
\newblock {Q-Bench}: A benchmark for general-purpose foundation models on
  low-level vision.
\newblock In B.~Kim, Y.~Yue, S.~Chaudhuri, K.~Fragkiadaki, M.~Khan, and Y.~Sun
  (eds.), \emph{International Conference on Learning Representations}, volume
  2024, pp.\  12547--12573, 2024.
\newblock URL
  \url{https://proceedings.iclr.cc/paper_files/paper/2024/file/363d4c97bb411e4b07612915b76c06ae-Paper-Conference.pdf}.

\bibitem[Yang et~al.(2025)Yang, Li, Yang, Zhang, Hui, Zheng, Yu, Gao, Huang,
  Lv, Zheng, Liu, Zhou, Huang, Hu, Ge, Wei, Lin, Tang, Yang, Tu, Zhang, Yang,
  Yang, Zhou, Zhou, Lin, Dang, Bao, Yang, Yu, Deng, Li, Xue, Li, Zhang, Wang,
  Zhu, Men, Gao, Liu, Luo, Li, Tang, Yin, Ren, Wang, Zhang, Ren, Fan, Su,
  Zhang, Zhang, Wan, Liu, Wang, Cui, Zhang, Zhou, and Qiu]{yang2025qwen3}
An~Yang, Anfeng Li, Baosong Yang, Beichen Zhang, Binyuan Hui, Bo~Zheng, Bowen
  Yu, Chang Gao, Chengen Huang, Chenxu Lv, Chujie Zheng, Dayiheng Liu, Fan
  Zhou, Fei Huang, Feng Hu, Hao Ge, Haoran Wei, Huan Lin, Jialong Tang, Jian
  Yang, Jianhong Tu, Jianwei Zhang, Jianxin Yang, Jiaxi Yang, Jing Zhou,
  Jingren Zhou, Junyang Lin, Kai Dang, Keqin Bao, Kexin Yang, Le~Yu, Lianghao
  Deng, Mei Li, Mingfeng Xue, Mingze Li, Pei Zhang, Peng Wang, Qin Zhu, Rui
  Men, Ruize Gao, Shixuan Liu, Shuang Luo, Tianhao Li, Tianyi Tang, Wenbiao
  Yin, Xingzhang Ren, Xinyu Wang, Xinyu Zhang, Xuancheng Ren, Yang Fan, Yang
  Su, Yichang Zhang, Yinger Zhang, Yu~Wan, Yuqiong Liu, Zekun Wang, Zeyu Cui,
  Zhenru Zhang, Zhipeng Zhou, and Zihan Qiu.
\newblock {Qwen3} technical report, 2025.
\newblock URL \url{https://arxiv.org/abs/2505.09388}.

\bibitem[Yin et~al.(2026)Yin, Yao, Cai, Chen, Wang, Yang, Su, and
  Zhao]{yin2026hopd}
Qixiang Yin, Huanjin Yao, Yuchen Cai, Jianghao Chen, Ziyi Wang, Min Yang, Fei
  Su, and Zhicheng Zhao.
\newblock {H-OPD}: Confidence aware heterogeneous multi-teacher multimodal
  on-policy distillation, 2026.
\newblock URL \url{https://arxiv.org/abs/2607.02592}.

\bibitem[Yoon et~al.(2026)Yoon, Yoon, Jang, Eom, Hong, Hasegawa-Johnson, Dai,
  Luo, and Yoo]{yoon2026decomposedopd}
Hee~Suk Yoon, Eunseop Yoon, Jaehyun Jang, SooHwan Eom, Ji~Woo Hong, Mark~A.
  Hasegawa-Johnson, Qi~Dai, Chong Luo, and Chang~D. Yoo.
\newblock Decomposed on-policy distillation for vision-language reasoning:
  Steering gradients for visual grounding.
\newblock In \emph{Forty-third International Conference on Machine Learning},
  2026.
\newblock URL \url{https://openreview.net/forum?id=Jb1YBPqyb7}.

\bibitem[Yue et~al.(2024)Yue, Ni, Zhang, Zheng, Liu, Zhang, Stevens, Jiang,
  Ren, Sun, Wei, Yu, Yuan, Sun, Yin, Zheng, Yang, Liu, Huang, Sun, Su, and
  Chen]{yue2024mmmu}
Xiang Yue, Yuansheng Ni, Kai Zhang, Tianyu Zheng, Ruoqi Liu, Ge~Zhang, Samuel
  Stevens, Dongfu Jiang, Weiming Ren, Yuxuan Sun, Cong Wei, Botao Yu, Ruibin
  Yuan, Renliang Sun, Ming Yin, Boyuan Zheng, Zhenzhu Yang, Yibo Liu, Wenhao
  Huang, Huan Sun, Yu~Su, and Wenhu Chen.
\newblock {MMMU}: A massive multi-discipline multimodal understanding and
  reasoning benchmark for expert {AGI}.
\newblock In \emph{Proceedings of the IEEE/CVF Conference on Computer Vision
  and Pattern Recognition (CVPR)}, pp.\  9556--9567, June 2024.
\newblock URL
  \url{https://openaccess.thecvf.com/content/CVPR2024/html/Yue_MMMU_A_Massive_Multi-discipline_Multimodal_Understanding_and_Reasoning_Benchmark_for_CVPR_2024_paper.html}.

\bibitem[Yue et~al.(2025)Yue, Zheng, Ni, Wang, Zhang, Tong, Sun, Yu, Zhang,
  Sun, Su, Chen, and Neubig]{yue2024mmmupro}
Xiang Yue, Tianyu Zheng, Yuansheng Ni, Yubo Wang, Kai Zhang, Shengbang Tong,
  Yuxuan Sun, Botao Yu, Ge~Zhang, Huan Sun, Yu~Su, Wenhu Chen, and Graham
  Neubig.
\newblock {MMMU}-pro: A more robust multi-discipline multimodal understanding
  benchmark.
\newblock In Wanxiang Che, Joyce Nabende, Ekaterina Shutova, and Mohammad~Taher
  Pilehvar (eds.), \emph{Proceedings of the 63rd Annual Meeting of the
  Association for Computational Linguistics (Volume 1: Long Papers)}, pp.\
  15134--15186, Vienna, Austria, July 2025. Association for Computational
  Linguistics.
\newblock ISBN 979-8-89176-251-0.
\newblock \doi{10.18653/v1/2025.acl-long.736}.
\newblock URL \url{https://aclanthology.org/2025.acl-long.736/}.

\bibitem[Zhang et~al.(2025)Zhang, Jiang, Zhang, Lin, Guo, Qiu, Zhou, Lu, Chang,
  Qiao, Gao, and Li]{zhang2024mathverse}
Renrui Zhang, Dongzhi Jiang, Yichi Zhang, Haokun Lin, Ziyu Guo, Pengshuo Qiu,
  Aojun Zhou, Pan Lu, Kai-Wei Chang, Yu~Qiao, Peng Gao, and Hongsheng Li.
\newblock {MATHVERSE}: Does your multi-modal {LLM} truly see the diagrams in
  visual math problems?
\newblock In Ale{\v{s}} Leonardis, Elisa Ricci, Stefan Roth, Olga Russakovsky,
  Torsten Sattler, and G{\"u}l Varol (eds.), \emph{Computer Vision -- ECCV
  2024}, volume 15066 of \emph{Lecture Notes in Computer Science}, pp.\
  169--186, Cham, 2025. Springer Nature Switzerland.
\newblock ISBN 978-3-031-73242-3.
\newblock \doi{10.1007/978-3-031-73242-3_10}.
\newblock URL \url{https://doi.org/10.1007/978-3-031-73242-3_10}.

\bibitem[Zhang et~al.(2024)Zhang, Wu, Zhang, Zhai, and
  Lin]{zhang2024qbenchplus}
Zicheng Zhang, Haoning Wu, Erli Zhang, Guangtao Zhai, and Weisi Lin.
\newblock {Q-Bench+}: A benchmark for multi-modal foundation models on
  low-level vision from single images to pairs.
\newblock \emph{IEEE Transactions on Pattern Analysis and Machine
  Intelligence}, 46\penalty0 (12):\penalty0 10404--10418, 2024.
\newblock \doi{10.1109/TPAMI.2024.3445770}.
\newblock URL \url{https://doi.org/10.1109/TPAMI.2024.3445770}.

\end{thebibliography}

\appendix
\section{Theoretical Analysis}
\label{app:proofs}
\label{sec:calibration-details}

This section establishes the properties of the teacher composition and
calibration rule in Section~\ref{sec:method}. We fix a student-generated
prefix and suppress its index $t$. The vocabulary $\V$ is finite, and
$q_L$, $q_V^I$, and $q_V^0$ are strictly positive distributions, as holds
for softmax distributions with finite logits. All logarithms are natural,
and $\Delta_{\V}$ denotes the probability simplex over $\V$.
We use the convention $0\log 0=0$ in KL divergences.

Define the raw visual residual, source KL, and target budget by
\begin{equation}
r(v):=\log\frac{q_V^I(v)}{q_V^0(v)},
\qquad d:=\KL(q_V^I\|q_V^0),
\qquad b:=\rho d,\quad \rho\geq0.
\label{eq:app-source-budget}
\end{equation}
For a finite coefficient $\alpha\geq0$, let
\begin{equation}
\begin{aligned}
Z(\alpha)&:=\sum_{v\in\V}q_L(v)e^{\alpha r(v)},
& A(\alpha)&:=\log Z(\alpha),\\
q_\alpha(v)&:=q_L(v)e^{\alpha r(v)-A(\alpha)},
& K(\alpha)&:=\KL(q_\alpha\|q_L).
\end{aligned}
\label{eq:partition}
\end{equation}
Thus exact calibration requires $K(\alpha^\star)=b$.
The results below concern the raw residual and full-vocabulary
distributions. Residual stabilization, restricted coefficient statistics,
and the sparse student update are described separately in
Appendix~\ref{sec:implementation}.

\subsection{Teacher Composition}
\label{app:composition-proof}

\begin{proposition}[Variational characterization]
For every finite $\alpha\geq0$, $q_\alpha$ is the unique maximizer of
\begin{equation}
\max_{q\in\Delta_{\V}}
\left\{\alpha\E_q[r]-\KL(q\|q_L)\right\}.
\label{eq:app-variational-objective}
\end{equation}
\end{proposition}

\begin{proof}
For any $q\in\Delta_{\V}$, substituting the definition of $q_\alpha$ gives
\begin{equation}
\KL(q\|q_\alpha)
=\KL(q\|q_L)-\alpha\E_q[r]+A(\alpha).
\label{eq:app-variational-identity}
\end{equation}
The objective in \eqref{eq:app-variational-objective} therefore
equals $A(\alpha)-\KL(q\|q_\alpha)$. Nonnegativity of KL, with equality
if and only if $q=q_\alpha$, proves the claim.
\end{proof}

\noindent\textbf{Equivalent logit-space form.}
Let $z_L$, $z_V^I$, and $z_V^0$ be logits for the three source
distributions. Each log probability equals its logit minus a
branch-specific log-normalizer. Consequently,
$\log q_L(v)+\alpha r(v)$ differs from
$z_L(v)+\alpha[z_V^I(v)-z_V^0(v)]$ only by a token-independent constant,
and
\begin{equation}
q_\alpha
=\operatorname{softmax}_{\V}
\bigl(z_L+\alpha[z_V^I-z_V^0]\bigr).
\label{eq:app-logit-composition}
\end{equation}
This is an algebraic identity for the source predictions; it does not
require the image-free prediction to equal a marginal over images.

\subsection{Exact KL Calibration: Existence and Uniqueness}
\label{sec:existence}

\begin{proposition}[Calibration range]
\label{prop:range}
If $r$ is nonconstant, $K$ is continuous and strictly increasing on
$[0,\infty)$, with
\begin{equation}
K(0)=0,\qquad
K'(\alpha)=\alpha\Var_{q_\alpha}[r]>0\quad(\alpha>0).
\label{eq:app-kl-derivative}
\end{equation}
Let $M:=\arg\max_{v\in\V}r(v)$ and
$q_L(M):=\sum_{v\in M}q_L(v)$. Then
\begin{equation}
\lim_{\alpha\to\infty}K(\alpha)
=K_\infty:=-\log q_L(M).
\label{eq:app-kl-limit}
\end{equation}
In particular, $K(\alpha^\star)=b$ has a unique finite positive solution
if and only if $0<b<K_\infty$.
\end{proposition}

\begin{proof}
Differentiating the finite sum defining $Z$ yields
\begin{equation}
A'(\alpha)=\E_{q_\alpha}[r],
\qquad A''(\alpha)=\Var_{q_\alpha}[r].
\label{eq:app-partition-derivatives}
\end{equation}
Since $\log(q_\alpha/q_L)=\alpha r-A(\alpha)$,
\begin{equation}
K(\alpha)=\alpha A'(\alpha)-A(\alpha),
\qquad K'(\alpha)=\alpha A''(\alpha).
\label{eq:app-kl-partition}
\end{equation}
For finite $\alpha$, $q_\alpha$ is strictly positive on $\V$.
A nonconstant $r$ therefore has positive variance under $q_\alpha$,
giving strict monotonicity for positive coefficients. Also,
$A(0)=0$ and $q_{\alpha=0}=q_L$, so $K(0)=0$.

To determine the limiting value, write $r_{\max}:=\max_v r(v)$ and
subtract it from every residual in the normalized expression for
$q_\alpha$. The factors $e^{\alpha(r(v)-r_{\max})}$ converge to one
on $M$ and to zero elsewhere. Hence
\begin{equation}
\lim_{\alpha\to\infty}q_\alpha(v)
=\begin{cases}
q_L(v)/q_L(M),&v\in M,\\
0,&v\notin M.
\end{cases}
\label{eq:app-limiting-teacher}
\end{equation}
On a finite vocabulary with strictly positive $q_L$, $\KL(q\|q_L)$ is
continuous in $q$, including at zero coordinates under the stated
convention. Applying this continuity to
\eqref{eq:app-limiting-teacher} gives
$K_\infty=-\log q_L(M)$. Continuity and strict monotonicity then prove
the existence and uniqueness statement.
\end{proof}

For nonconstant $r$, the zero-budget solution is uniquely $\alpha^\star=0$.
The upper endpoint $K_\infty$ is approached only as $\alpha\to\infty$;
no finite coefficient attains it, and larger budgets cannot be attained
along this composition family. If $r$ is constant, $q_\alpha=q_L$ and
$K(\alpha)=0$ for every coefficient. For the raw residual in
\eqref{eq:app-source-budget}, constancy implies
$q_V^I=e^r q_V^0$, and normalization forces $r=0$ and $d=0$.
We select $\alpha^\star=0$ in this degenerate case.

\noindent\textbf{KL-constrained interpretation.}
For $0<b<K_\infty$, the exactly calibrated teacher also uniquely solves
\begin{equation}
\max_{q\in\Delta_{\V}}\ \E_q[r]
\quad\text{subject to}\quad \KL(q\|q_L)\leq b.
\label{eq:app-constrained}
\end{equation}
Indeed, \eqref{eq:app-variational-identity} and
$K(\alpha^\star)=b$ imply, for every feasible $q$,
\begin{equation}
\begin{aligned}
\alpha^\star\bigl(\E_q[r]-\E_{q_{\alpha^\star}}[r]\bigr)
&=\KL(q\|q_L)-b-\KL(q\|q_{\alpha^\star})\\
&\leq-\KL(q\|q_{\alpha^\star})\leq0.
\end{aligned}
\label{eq:app-constrained-proof}
\end{equation}
Since $\alpha^\star>0$, no feasible distribution has a greater expected
residual, and equality requires $q=q_{\alpha^\star}$. For $b=0$, the
only feasible distribution is $q_L$. This characterization holds for a
prescribed budget; the source-referenced choice $b=\rho d$ specifies
that budget rather than following from constrained optimality.

\noindent\textbf{Finite coefficient caps.}
For nonconstant $r$ and a finite cap $\alpha_{\max}>0$, the attainable
budgets with $0\leq\alpha\leq\alpha_{\max}$ form the closed interval
$[0,K(\alpha_{\max})]$. A bounded exact reference can therefore be
defined by the unique solution of
\begin{equation}
K(\alpha_{\mathrm{cap}})
=\min\{b,K(\alpha_{\max})\},
\qquad 0\leq\alpha_{\mathrm{cap}}\leq\alpha_{\max}.
\label{eq:app-capped-reference}
\end{equation}
When $b>K(\alpha_{\max})$, this reference returns the cap and falls
short of the requested budget. Clipping a locally approximated
coefficient need not give this bounded exact solution.

\subsection{Local Approximation and Error}
\label{sec:alpha0}

For nonconstant $r$, define
$\mu:=\E_{q_L}[r]$, $\sigma_L^2:=\Var_{q_L}[r]>0$, and
$\kappa_{3,L}:=\E_{q_L}[(r-\mu)^3]$.
The finite-vocabulary log-partition function is analytic near zero,
with expansion
\begin{align}
A(\alpha)
&=\mu\alpha+\tfrac12\sigma_L^2\alpha^2
  +\tfrac16\kappa_{3,L}\alpha^3+O(\alpha^4),
\label{eq:app-cumulant-expansion}\\
K(\alpha)
&=\tfrac12\sigma_L^2\alpha^2
  +\tfrac13\kappa_{3,L}\alpha^3+O(\alpha^4).
\label{eq:app-kl-expansion}
\end{align}
The second identity follows by substituting the first into
\eqref{eq:app-kl-partition}. Matching its quadratic term to
$b=\rho d$ gives the coefficient used in the idealized local rule,
\begin{equation}
\widehat\alpha
=\sqrt{\frac{2b}{\sigma_L^2}}
=\sqrt{\frac{2\rho\KL(q_V^I\|q_V^0)}{\Var_{q_L}[r]}}.
\label{eq:app-local-coefficient}
\end{equation}
We set $\widehat\alpha=0$ when $\sigma_L^2=0$; for the raw residual
under the stated assumptions, this also implies $d=0$.

\noindent\textbf{Local accuracy.}
To state the approximation precisely, hold $q_L$ and a nonconstant $r$
fixed, introduce an auxiliary budget $0<\delta<K_\infty$, and let
$\alpha^\star(\delta)$ solve $K(\alpha^\star(\delta))=\delta$.
Writing $\widehat\alpha(\delta):=\sqrt{2\delta/\sigma_L^2}$, we obtain
\begin{align}
\alpha^\star(\delta)
&=\sqrt{\frac{2\delta}{\sigma_L^2}}+O(\delta),
\label{eq:app-inverse-error}\\
K(\widehat\alpha(\delta))-\delta
&=\frac{\kappa_{3,L}}{3}
  \left(\frac{2\delta}{\sigma_L^2}\right)^{3/2}
  +O(\delta^2),\qquad \delta\downarrow0.
\label{eq:app-budget-error}
\end{align}
For the first statement, strict monotonicity implies
$\alpha^\star(\delta)\to0$ as $\delta\downarrow0$.
\eqref{eq:app-kl-expansion} then gives
$\delta=\tfrac12\sigma_L^2[\alpha^\star(\delta)]^2
(1+O(\alpha^\star(\delta)))$.
Taking the nonnegative square root first establishes
$\alpha^\star(\delta)=O(\sqrt\delta)$ and then yields
\eqref{eq:app-inverse-error}. The second statement follows by
substituting $\widehat\alpha(\delta)$ into the same expansion.

These are local statements for a fixed prior and residual; their
constants may depend on both. Evaluating the approximation at the
observed budget $b=\rho d$ does not ensure exact matching or uniformly
small error across prefixes. The cubic term also shows that the
approximation may overshoot or undershoot the budget. If $b\geq K_\infty$,
the formula still defines a finite coefficient when $\sigma_L^2>0$,
but there is no finite exactly calibrated coefficient to approximate.

\noindent\textbf{Teacher discrepancy along a common composition path.}
For two finite coefficients $0\leq a_1\leq a_2$ with the same prior
and residual,
\begin{align}
\KL(q_{a_2}\|q_{a_1})
&=(a_2-a_1)A'(a_2)-A(a_2)+A(a_1)\nonumber\\
&=\int_{a_1}^{a_2}(u-a_1)A''(u)\,du\nonumber\\
&\leq\int_{a_1}^{a_2}uA''(u)\,du
 =K(a_2)-K(a_1).
\label{eq:error-bound}
\end{align}
The inequality uses $a_1\geq0$ and $A''(u)\geq0$.
If one coefficient is exactly calibrated to $b$, the right-hand side
equals the absolute budget error of the other coefficient. The bound
is directional: the KL is taken from the teacher with the larger
coefficient to the one with the smaller coefficient. It requires the
same prior, residual, and normalized support for both teachers.
It therefore does not directly compare a raw-residual teacher with a
teacher constructed using a different stabilized residual.

\subsection{Special Cases and Invariances}
\label{app:special-cases}

\noindent\textbf{Zero visual budget.}
If $\rho=0$, selecting $\alpha=0$ gives $q_\alpha=q_L$ regardless of
the source KL. If $q_V^I=q_V^0$, then $r=0$ and $d=0$, and every
coefficient yields $q_L$; our zero-coefficient convention selects a
single representative.

\noindent\textbf{Self-prior recovery.}
Suppose $q_L=q_V^0$ and $\rho=1$. For the raw residual,
\begin{equation}
q_L(v)e^{r(v)}=q_V^I(v),
\qquad Z(1)=1,
\qquad q_{\alpha=1}=q_V^I.
\label{eq:app-self-prior}
\end{equation}
Thus $K(1)=\KL(q_V^I\|q_V^0)=d=b$.
If $d>0$, Proposition~\ref{prop:range} implies that the unique exact
coefficient is $\alpha^\star=1$. If $d=0$, the two source distributions
coincide and the zero-coefficient convention recovers the same teacher.
This identity concerns exact calibration: the local coefficient
$\sqrt{2d/\Var_{q_V^0}[r]}$ need not equal one. Residual modification
can also prevent recovery, and a cap below one excludes the exact
solution when $d>0$.

\noindent\textbf{Affine transformations of the residual.}
For a fixed budget $b$, let $r'(v)=c\,r(v)+\beta$, where $c>0$ and
$\beta$ is independent of $v$. If $q'_a$ denotes composition using
$r'$, normalization gives
\begin{equation}
q'_a(v)
=\frac{q_L(v)e^{a(cr(v)+\beta)}}
       {\sum_u q_L(u)e^{a(cr(u)+\beta)}}
=q_{ca}(v).
\label{eq:app-affine-invariance}
\end{equation}
A token-independent shift therefore leaves the teacher and its KL
displacement unchanged at every coefficient. More generally, whenever
the exact matching coefficient is unique, replacing $r$ by $r'$
changes it to $\alpha^\star/c$ and preserves the calibrated teacher.
The unregularized, uncapped local rule has the same property because
$\Var_{q_L}[r']=c^2\Var_{q_L}[r]$. This statement holds with the budget
fixed. Changing source-model temperatures can change the source KL
itself, while a fixed variance regularizer or coefficient cap can
break the coefficient-rescaling property.

\subsection{Loss Decomposition}
\label{app:loss-decomposition}

At the fixed prefix, let $p_\theta$ and $p_\theta^0$ be the student's
strictly positive image-conditioned and image-free distributions,
and define $r_\theta^S:=\log p_\theta-\log p_\theta^0$.
The coefficient $\alpha$ and teacher sources are fixed during the
student update. Substitution into the full-vocabulary reverse KL gives
\begin{align}
\KL(p_\theta\|q_\alpha)
&=\E_{p_\theta}
  [\log p_\theta-\log q_L-\alpha r]+A(\alpha)\nonumber\\
&=\underbrace{\E_{p_\theta}
  \left[\log\frac{p_\theta^0}{q_L}\right]}
  _{\text{language-prior contribution}}
 +\underbrace{\E_{p_\theta}[r_\theta^S-\alpha r]}
  _{\text{visual-residual contribution}}
 +A(\alpha).
\label{eq:app-loss-decomposition}
\end{align}
This identity holds for any fixed coefficient, including the locally
approximated coefficient. The first expectation is taken under
$p_\theta$, not $p_\theta^0$, so it is generally not
$\KL(p_\theta^0\|q_L)$. Neither contribution is individually guaranteed
to be nonnegative. The decomposition describes two contributions to
a single objective rather than two separately minimized divergences.
The image-free student distribution is introduced only for this
identity; evaluating the original reverse KL does not require an
additional image-free student forward pass.

Because the prefix, teacher sources, and coefficient are fixed,
$A(\alpha)$ has zero gradient with respect to $\theta$ during the
update. Derivatives through $\log p_\theta^0$ cancel between the two
contributions when their sum is differentiated consistently. These
claims apply to the fixed-prefix objective in
\eqref{eq:opd-loss}, where sampled prefixes are held fixed,
and do not differentiate through the rollout sampling distribution.

\noindent\textbf{Relation to the detached visual target.}
We now relate the visual-residual term in the decomposition above to
the detached visual target used in Decomposed
OPD~\citep{yoon2026decomposedopd}. Let $\bar\theta$ denote the student
parameters at which the target is constructed. The target combines the
student's image-free prior $p_{\bar\theta}^0$ with the teacher's visual
residual $r$, with the prior held fixed during differentiation:
\begin{equation}
q_{\mathrm{vis},\bar\theta}(v)
:=\frac{p_{\bar\theta}^0(v)e^{r(v)}}
        {\sum_u p_{\bar\theta}^0(u)e^{r(u)}},
\qquad
B_{\bar\theta}:=\log\sum_u p_{\bar\theta}^0(u)e^{r(u)}.
\label{eq:app-detached-visual-target}
\end{equation}
The corresponding KL loss is
\begin{equation}
\KL(p_\theta\|q_{\mathrm{vis},\bar\theta})
=\E_{p_\theta}
 [\log p_\theta-\log p_{\bar\theta}^0-r]+B_{\bar\theta}.
\label{eq:app-detached-visual-loss}
\end{equation}
At $\theta=\bar\theta$, its value is
$\E_{p_{\bar\theta}}[r_{\bar\theta}^S-r]+B_{\bar\theta}$,
which is the visual-residual contribution in
\eqref{eq:app-loss-decomposition} at $\alpha=1$, up to
normalization. However, this value correspondence does not by itself
imply gradient equivalence. In particular, if
$F(\theta):=\E_{p_\theta}[r_\theta^S-r]$ uses an undetached
$p_\theta^0$, differentiation yields
\begin{equation}
\begin{aligned}
\left.\nabla_\theta\KL(p_\theta\|q_{\mathrm{vis},\bar\theta})
\right|_{\theta=\bar\theta}
&=\left.\nabla_\theta F(\theta)\right|_{\theta=\bar\theta}\\
&\quad+\E_{p_{\bar\theta}}
 \left[\left.\nabla_\theta\log p_\theta^0
 \right|_{\theta=\bar\theta}\right].
\end{aligned}
\label{eq:app-visual-gradient-difference}
\end{equation}
The extra expectation need not vanish because it is taken under
$p_{\bar\theta}$. The target holds the student's prior fixed, whereas
$F(\theta)$ differentiates through it. Our reproduction settings and
results are reported in Appendix~\ref{app:decomposed-reproduction}.

\section{Practical Teacher Construction and Student Update}
\label{app:lego_implementation}
\label{app:implementation}
\label{sec:implementation}

We describe the numerical safeguards used in visual-residual
construction, coefficient estimation, and teacher normalization.
All expert branches are frozen and score the same student-generated
response prefix over a shared output vocabulary. The image-free branch
omits the image and its placeholder while retaining the system prompt,
question, and generated text. Below, we suppress the prefix index $t$ and write
$q_L=\operatorname{softmax}(z_L)$ and
$q_V^c=\operatorname{softmax}(z_V^c)$ for $c\in\{I,0\}$.

\subsection{Visual Residual Stabilization}
\label{app:residual-stabilization}

Ratios involving very small image-free probabilities can amplify
errors in tail predictions. We therefore floor low image-free logits,
while protecting selected tokens favored by the image.
Let $\mathcal N_m(q)$ denote the smallest descending-probability
nucleus with cumulative mass at least $m$. The floor is the boundary
logit of the original image-free distribution's $0.999$-mass nucleus,
\begin{equation}
f:=\min_{v\in\mathcal N_{0.999}(q_V^0)}z_V^0(v).
\label{eq:impl-null-floor}
\end{equation}
A token is protected if it lies in the original image-conditioned
$0.99$-mass nucleus, its probability increases with the image, and
its original image-free logit lies below the floor:
\begin{equation}
P:=\left\{v\in\mathcal N_{0.99}(q_V^I):
q_V^I(v)>q_V^0(v),\ z_V^0(v)<f\right\}.
\label{eq:impl-protected-set}
\end{equation}
The effective image-free logits and stabilized residual are
\begin{equation}
\begin{aligned}
\widetilde z_V^0(v)
&:=\begin{cases}
z_V^0(v),&v\in P,\\
\max\{z_V^0(v),f\},&v\notin P,
\end{cases}\\
\widetilde r(v)&:=z_V^I(v)-\widetilde z_V^0(v).
\end{aligned}
\label{eq:impl-stabilized-residual}
\end{equation}
Flooring reduces the residual for unprotected tokens with very low
image-free probability. Protection preserves the original contrast
for tokens satisfying all three conditions in
\eqref{eq:impl-protected-set}.

We then center the residual by subtracting its maximum,
\begin{equation}
u(v):=\widetilde r(v)-\max_{a\in\V}\widetilde r(a).
\label{eq:impl-centered-residual}
\end{equation}
This constant shift preserves the residual variance and the normalized
teacher, while keeping the transferred residual nonpositive.
The original branch probabilities are retained for support selection
and source-KL estimation; flooring is applied only when constructing
the transferred residual.

\subsection{Calibration Coefficient Estimation}
\label{app:practical-calibration}

To estimate the coefficient from high-probability candidates, we form
the union of the three original teacher nuclei,
\begin{equation}
S:=\mathcal N_m(q_L)\cup\mathcal N_m(q_V^I)
\cup\mathcal N_m(q_V^0).
\label{eq:union}
\end{equation}
We use a calibration support mass of $m=0.95$;
Appendix~\ref{app:support-mass-ablation} compares it with $m=0.99$.
Each branch is renormalized on the common set $S$: for
$q\in\{q_L,q_V^I,q_V^0\}$, define
$q^S(v):=q(v)/\sum_{a\in S}q(a)$ for $v\in S$.
The coefficient uses the original image/no-image contrast and the
variance of the stabilized residual under the receiving language prior,
\begin{equation}
\begin{aligned}
d_S&:=\KL(q_V^{I,S}\|q_V^{0,S}),\\
V_S&:=\Var_{q_L^S}[\widetilde r].
\end{aligned}
\label{eq:impl-coefficient-statistics}
\end{equation}

The union retains candidates assigned high probability by any of the
three branches. For every branch $q$, $q(S)\geq m$ and
$\operatorname{TV}(q,q^S)=1-q(S)\leq1-m$, where $q^S$ is extended
by zero outside $S$. Thus $m$ controls the excluded model probability
mass. We use $d_S$ as a conditional calibration statistic to reduce
dependence on low-probability likelihood ratios; high retained mass
alone does not guarantee closeness to the full-vocabulary source KL.

We use $\rho=1$ by default, regularize the denominator, and cap the
coefficient:
\begin{equation}
\widehat\alpha_S
:=\min\!\left\{\alpha_{\max},
\sqrt{\frac{2\rho d_S}{V_S+\epsilon}}\right\},
\qquad \epsilon=10^{-8},\quad \alpha_{\max}=6.
\label{eq:practical-alpha}
\end{equation}
The regularizer makes the expression defined when the estimated
variance is zero, and the cap limits amplification when that variance
is small. This cap is a practical safeguard against large
coefficients, not a way to make an unattainable KL target attainable.
We use the same $\alpha_{\max}=6$ across adaptive training variants.
A zero estimated source KL gives a zero coefficient.
We compute the coefficient statistics in FP32 and clamp numerical
KL and variance estimates to be nonnegative. As a numerical safeguard,
NaN and negative-infinite coefficients are replaced by zero and
positive-infinite coefficients by $\alpha_{\max}$ before clamping
to $[0,\alpha_{\max}]$.

This is a regularized, capped second-order estimate, rather than an
exact solution of the KL-budget equation. Support restriction and
residual stabilization also affect the approximation, so the practical
teacher need not attain the full-vocabulary displacement
$\rho\KL(q_V^I\|q_V^0)$. Its induced KL displacement is evaluated in
Appendix~\ref{app:calibration-diagnostics}.

\subsection{Teacher Normalization and Student Update}
\label{app:sparse-opd}

The final teacher is normalized over the entire output vocabulary,
\begin{equation}
q_T(v)
:=\frac{\exp\{z_L(v)+\widehat\alpha_S u(v)\}}
{\sum_{a\in\V}\exp\{z_L(a)+\widehat\alpha_S u(a)\}}.
\label{eq:impl-full-teacher}
\end{equation}
Thus $S$ restricts coefficient estimation, not teacher support.

The variational characterization in
Appendix~\ref{app:composition-proof} applies to any finite evidence
score. Consequently, the practical teacher in
\eqref{eq:impl-full-teacher} uniquely maximizes
$\widehat\alpha_S\E_q[u]-\KL(q\|q_L)$ over $q\in\Delta_{\V}$.
Stabilization changes the transferred score and its coefficient while
preserving this generalized Bayesian form.

We evaluate the full-vocabulary log-normalizer using FP32
log-sum-exp reductions over vocabulary chunks, combined with
log-add-exp. This preserves the full normalization without
materializing a full FP32 tensor of composed teacher logits.
The distillation temperature is $1$.

\noindent\textbf{Adopted top-16 OPD update.}
We follow the student-top-16 OPD setting of
\citet{li2026rethinkingopd}, using the composed $q_T$ as the teacher.
For reproducibility, we specify the support and gradient placement.
Let $\bar\theta$ denote the current scoring parameters and let
$p_\theta$ be the image-conditioned student distribution at the fixed
prefix. We select $k=16$ tokens using detached student scores,
\begin{equation}
U:=\operatorname{TopK}_{v\in\V}p_{\bar\theta}(v),
\qquad
w(v):=\frac{p_{\bar\theta}(v)}
{\sum_{a\in U}p_{\bar\theta}(a)},\quad v\in U.
\label{eq:impl-student-support}
\end{equation}
Only the weights are renormalized on $U$. The student and teacher
log probabilities retain their full-vocabulary normalizers.
The token-level surrogate defining the update is
\begin{equation}
\widehat{\mathcal L}(\theta)
:=\sum_{v\in U}\sg\!\left[
w(v)\log\frac{p_{\bar\theta}(v)}{q_T(v)}
\right]\log p_\theta(v).
\label{eq:sparse}
\end{equation}
The weights and log ratios are recomputed at the current scoring
parameters and detached before differentiation. Gradients flow
only through the final student log-probability factor; the prefix,
expert predictions, teacher target, and selected support are fixed.

The teacher-derived support $S$ is used only for coefficient
estimation, whereas the student-derived support $U$ selects the terms
in this update. Both student and teacher log probabilities remain
normalized over $\V$.

\newpage
\section{Experimental Details}
\label{app:experimental-details}
 
\subsection{Models and Training}
\label{app:training_settings}
 
Table~\ref{tab:lego_settings} summarizes the reference training
configuration and the sampled-decoding settings for reasoning evaluation.
Perception benchmarks use the separate settings in
Appendix~\ref{app:perception-eval}.
Baseline-specific adaptations are described in
Appendix~\ref{app:baseline_impl}.
 
\begin{table}[H]
\centering
\caption{
Training and evaluation settings used in our experiments.
The training configuration corresponds to the main LEGO-OPD setup and
serves as the reference configuration for the baseline comparisons.
Baselines retain these settings whenever compatible with their defining
formulations, with method-specific deviations described in
Appendix~\ref{app:baseline_impl}.
The reasoning-evaluation settings apply to sampled decoding and are
held fixed across compared methods. Perception settings are given in
Appendix~\ref{app:perception-eval}.
}
\label{tab:lego_settings}
 
\footnotesize
\renewcommand{\arraystretch}{1.10}
\setlength{\tabcolsep}{3.5pt}
 
\begin{tabular*}{\linewidth}{@{\extracolsep{\fill}}
    p{0.20\columnwidth}
    p{0.39\columnwidth}
    p{0.31\columnwidth}
@{}}
\toprule
\textbf{Category} & \textbf{Hyperparameter} & \textbf{Value} \\
\midrule
 
\multirow{3}{*}{\textbf{Models}}
& Student & Qwen3-VL-2B-Instruct \\
& Language expert & Qwen3-4B-Instruct-2507 \\
& Grounding expert & Qwen3-VL-4B-Instruct \\
\midrule
 
\multirow{8}{*}{\textbf{Training}}
& Training data & MMFineReason, 54,840 examples \\
& Student update & Full-parameter fine-tuning \\
& Training steps & 200 \\
& Responses per update & 128 \\
& Optimizer & AdamW \\
& Learning rate & $1\times10^{-6}$ \\
& Weight decay & $0.01$ \\
& LR schedule & Constant, no warmup \\
\midrule
 
\multirow{7}{*}{\textbf{Training rollout}}
& Maximum prompt length & 12,384 tokens \\
& Maximum response length & 12,384 tokens \\
& Temperature & $1.0$ \\
& Top-$p$ & $1.0$ \\
& Top-$k$ & Unrestricted \\
& Min-$p$ & $0$ \\
& Repetition penalty & $1.0$ \\
\midrule
 
\multirow{6}{*}{\textbf{LEGO-OPD}}
& OPD support & Student top 16 \\
& KL-budget scale & $\rho=1$ \\
& Calibration support mass & $0.95$ \\
& Coefficient cap & $\alpha_{\max}=6$ \\
& Null-floor mass & $0.999$ \\
& Image-protection mass & $0.99$ \\
\midrule
 
\multirow{8}{*}{\textbf{Reasoning eval.}}
& Temperature & $0.9$ \\
& Top-$p$ & $0.95$ \\
& Top-$k$ & $100$ \\
& Min-$p$ & $0.01$ \\
& Repetition penalty & $1.03$ \\
& Maximum response length & 12,384 tokens \\
& Stop sequence & \texttt{</answer>} \\
& Inference repetitions & 5 (seeds 42--46) \\
\bottomrule
\end{tabular*}
\end{table}
 
\noindent\textbf{Training.}
We fully fine-tune the student while keeping the language and grounding
experts frozen. At each update, the current student generates on-policy
responses, which are then scored by the expert branches.
We do not use teacher-generated responses, ground-truth answers, or reference
solutions as training targets.
 
We use the prompt template shown in Figure~\ref{fig:prompt-template}
throughout training.
The language expert receives the textual problem only, while the grounding
expert is evaluated with and without the image under the same generated
response prefix.
LEGO-OPD uses the student-top-16 surrogate together with the practical
calibration and residual stabilization in Appendix~\ref{app:implementation}.
 
\subsection{Benchmarks and Metrics}
\label{app:benchmarks}
We evaluate multimodal reasoning on MATH-Vision, MathVerse, MathVista,
MMMU-Pro, RBench-M, and the multimodal split of OlympiadBench, and
text-only reasoning on MATH-500, MathVerse Text Only, MMLU-Pro, RBench-T,
and the text-only split of OlympiadBench. Visual perception is evaluated
on BLINK~\citep{fu2024blink} and QBench1,2~\citep{wu2023qbench,zhang2024qbenchplus}.
For QBench1,2, we use the perception-focused Group A questions, namely
the A1 perception questions from the development splits (1,495 from
QBench 1 and 1,000 from QBench 2), and report accuracy over the pooled
2,495 questions. The six multimodal and five text-only reasoning scores
are averaged within their respective groups. BLINK and QBench1,2 are
reported separately and excluded from both reasoning averages.
 
\subsection{Evaluation Protocol}
\label{app:evaluation}
 
Reasoning and perception benchmarks follow separate protocols.
Reasoning benchmarks use sampled decoding with the evaluation settings
in Table~\ref{tab:lego_settings} (Appendix~\ref{app:reasoning-eval}).
Perception benchmarks use greedy direct-answer generation with the
settings in Table~\ref{tab:perception_settings}
(Appendix~\ref{app:perception-eval}).
 
\subsubsection{Reasoning Benchmarks}
\label{app:reasoning-eval}
 
\noindent\textbf{Inputs.}
All reasoning benchmarks except MMLU-Pro use the shared system prompt in
Figure~\ref{fig:prompt-template}. The user prompt contains each question
in its benchmark's native format, together with any benchmark-specific
instructions. MathVerse and R-Bench add their own instructions to the
question, and OlympiadBench uses its multimodal or text-only prompt
template. MMLU-Pro uses category-conditioned five-shot prompts without a
system prompt.
 
\noindent\textbf{Decoding and reporting.}
We sample five responses per question with evaluation seeds 42--46 under
the reasoning-evaluation settings in Table~\ref{tab:lego_settings}.
The main text reports the mean accuracy over these samples, and
Appendix~\ref{app:supplementary-experiments} reports the same results
with standard deviations (mean $\pm$ SD).
Appendix~\ref{app:greedy-results} additionally reports greedy-decoding
results. These SDs reflect sampling variability across evaluation seeds
for each trained checkpoint. Benchmark averages are unweighted means and
are reported as means only.
 
\noindent\textbf{Answer extraction and scoring.}
For multiple-choice tasks, predictions are first mapped to answer
options by deterministic rules. Only when this mapping remains
unresolved do we use \texttt{gpt-5.6-luna} as a gold-blind semantic
option mapper, with the prompt in Figure~\ref{fig:prompt-luna}.
The mapper receives only the extracted final answer and the candidate
options.
 
MMLU-Pro is scored by a dedicated deterministic A--J evaluator.
OlympiadBench uses its benchmark-specific mathematical answer
normalization and automatic scorer, and MathVista and MATH-500 use their
task-specific answer normalization. MathVerse uses its dedicated answer
extraction and correctness-judging pipeline, calling
\texttt{gpt-5.6-luna} when an LLM-based judgment is required.
 
\subsubsection{Perception Benchmarks}
\label{app:perception-eval}
 
\noindent\textbf{Inputs and decoding.}
BLINK uses its dataset-provided prompt, and QBench1,2 uses the question
followed by its A--D choices. We append the instruction listed in
Table~\ref{tab:perception_settings}, use the Qwen3-VL processor and chat
template without a system prompt, and generate one response per question
with greedy decoding (temperature $0$) and seed 42.
 
\noindent\textbf{Scoring.}
The option letter is extracted from each response by deterministic rules
and compared with the gold letter by exact match. Responses without an
extractable letter are counted as incorrect. For BLINK, the letter is
taken from explicit answer patterns or a standalone final choice,
whereas for QBench1,2 it is taken from the beginning of the response.
BLINK accuracy is computed over all 1,901 validation questions, and
QBench1,2 accuracy over the 2,495 pooled Group A questions.
 
\begin{table}[H]
\centering
\caption{Settings for direct-answer perception evaluation. Both
benchmarks use one greedy response per question (temperature $0$,
seed 42) and no system prompt. Max.\ tokens is the maximum number of
generated tokens.}
\label{tab:perception_settings}
\footnotesize
\renewcommand{\arraystretch}{1.15}
\setlength{\tabcolsep}{4pt}
\begin{tabular*}{\linewidth}{@{\extracolsep{\fill}}
    >{\raggedright\arraybackslash}p{0.14\linewidth}
    >{\raggedright\arraybackslash}p{0.12\linewidth}
    >{\raggedright\arraybackslash}p{0.07\linewidth}
    >{\raggedright\arraybackslash}p{0.34\linewidth}
    >{\raggedright\arraybackslash}p{0.21\linewidth}
@{}}
\toprule
\textbf{Benchmark} & \textbf{Questions} & \textbf{Max. tokens} & \textbf{Appended instruction} & \textbf{Scoring} \\
\midrule
BLINK validation & 1,901 & 128 & ``Reply with only the letter of the correct choice (A, B, C, ...). Do not explain.'' & Rule-based letter extraction and exact match, accuracy over all questions \\
QBench1,2 Group A (A1 perception, dev) & 2,495 (1,495 + 1,000) & 16 & ``Reply with only the letter of the correct choice (A, B, C, or D). Do not explain.'' & Leading-letter extraction and exact match, accuracy over all pooled questions \\
\bottomrule
\end{tabular*}
\end{table}
 
\subsection{Prompt Templates}
\label{app:prompts}
 
Figures~\ref{fig:prompt-template} and~\ref{fig:prompt-luna} give the QA
template used for training and reasoning evaluation and the semantic
option-mapping prompt. Benchmark-specific formats and the separate
perception settings are described in
Appendices~\ref{app:reasoning-eval} and~\ref{app:perception-eval}.
 
\begin{figure}[!h]
\small
\centering
\begin{tcolorbox}[
    colback=gray!5!white,
    colframe=black,
    title=\textbf{Common QA Prompt},
    boxrule=0.6pt,
    arc=1.5mm,
    left=7pt,
    right=7pt,
    top=6pt,
    bottom=6pt
]
 
\textbf{System Prompt:}\quad
Please reason step by step, and put your final answer within
\texttt{<answer>\textbackslash boxed\{your final answer\}</answer>}.
 
\vspace{0.6em}
 
\textbf{User Prompt:}\quad
\{Question and answer choices\}
 
\end{tcolorbox}
 
\caption{
Common QA prompt used for training and reasoning evaluation.
Benchmark-specific input formatting is preserved during evaluation.
MMLU-Pro uses its native five-shot prompt without the shared system prompt.
Perception settings are given in Appendix~\ref{app:perception-eval}.
}
\label{fig:prompt-template}
\end{figure}
 
\begin{figure}[!h]
\small
\centering
\begin{tcolorbox}[
    colback=gray!5!white,
    colframe=black,
    title=\textbf{Semantic Option-Mapping Prompt for Scoring},
    boxrule=0.6pt,
    arc=1.5mm,
    left=7pt,
    right=7pt,
    top=6pt,
    bottom=6pt
]
 
\textbf{User Prompt:}\quad
Map the provided final answer to an option, or null. This is ONLY answer interpretation, not solving the original question. No correct answer key is supplied. \\
Rules: \\
- A clear standalone option letter selects that option. \\
- A clear final value or description can select the unique semantically equivalent option. \\
- If BOTH a letter and answer content are stated, they MUST agree with that option's content.  Otherwise output null. Example: B:52 is null when B contains 64, even if another option contains 52.\\
- A letter plus reasoning can be interpreted only when the stated answer content is consistent; do not evaluate unrelated intermediate reasoning or solve a question.\\
- Ignore Markdown, answer tags, boxing, harmless punctuation, and equivalent mathematical notation.\\
- If the supplied final regions conflict, contain competing answers, or cannot be mapped with certainty, output null.\\
- Do not infer an option from mathematical knowledge of what the answer SHOULD be.
Return ONLY a JSON object \{"option": "A"\} (using the selected letter) or \{"option": null\}.\\
 
Data:
\{options and model output\}
 
\end{tcolorbox}
 
\caption{
Prompt used by GPT-5.6 Luna for gold-blind semantic option mapping during scoring.
The judge receives only the extracted final answer and candidate options,
without the original question, image, or reference answer.
}
\label{fig:prompt-luna}
\end{figure}
 
\subsection{Baseline Implementations}
\label{app:baseline_impl}
 
We use the LEGO-OPD training setup in Table~\ref{tab:lego_settings}
whenever compatible with each baseline. Unless otherwise noted, we retain
the objectives and method-specific hyperparameters reported in the original
papers, while adapting the model configuration, data, and training budget to
our experimental setting.
 
\noindent\textbf{Standard OPD.}
Standard OPD distills directly from a frozen Qwen3-VL-4B-Instruct teacher
conditioned on the image and question. We use the same MMFineReason data,
rollout configuration, and student-top-16 OPD surrogate as LEGO-OPD.
 
\noindent\textbf{VA-OPD.}
VA-OPD retains its multi-response training scheme, generating four responses
per prompt and therefore using 32 unique prompts per update under the common
128-response budget. Following the reported configuration, the VLM teacher
scores each response with the original and pixelated image to estimate visual
importance. We use a degradation ratio of $0.1$, assign the highest-advantage
20\% of token positions to the high-importance group, and give equal weight
to the high- and low-importance groups. Training uses full-vocabulary reverse
KL rather than the student-top-16 surrogate.
 
\noindent\textbf{M-OPD.}
M-OPD was originally formulated for domain-routed multi-teacher
distillation over text domains. We adapt it to the multimodal setting by
routing examples from our multimodal training data (the 54,840-example
MMFineReason subset in Section~\ref{sec:experimental-setup}) to the VLM
teacher and DeepMath-103K~\citep{he2025deepmath103k} examples to the LLM
teacher at a $1{:}1$ ratio. Following M-OPD's policy-gradient
formulation, we use sampled-token distillation with the reported
advantage clipping magnitude of $5$ instead of a top-$K$ objective.
 
\noindent\textbf{H-OPD.}
H-OPD combines a VLM teacher conditioned on the image and question with an
LLM teacher conditioned on the released description-augmented input.
Following its reported configuration, each teacher distribution is
normalized over its top-eight tokens, and the two distributions are combined
using entropy-based confidence weights with temperature $0.5$.
The student is trained with reverse KL over the union of the two teacher
supports. This support contains at most 16 tokens and is distinct from the
student-top-16 support used by Standard OPD and LEGO-OPD.

\section{Analysis of Source-Referenced Calibration}
\label{app:practical-validation}

This section analyzes the calibration rule in three parts.
Appendix~\ref{app:toy_example} uses a three-token example to show why a
fixed coefficient fails when the language prior differs from the
grounding expert's image-free prior.
Appendices~\ref{app:calibration-results}, \ref{app:zero-budget-controls},
and~\ref{app:support-mass-ablation} vary one component of the calibration
in training: the coefficient rule and KL-budget scale, the zero-budget
limit, and the calibration support mass.
Appendices~\ref{app:calibration-diagnostics}
and~\ref{app:recomposed-training} evaluate the practical calibration rule,
first against numerical KL matching and then through distillation from a
recomposed VLM teacher.

\subsection{A Three-Token Calibration Example}
\label{app:toy_example}

Section~\ref{sec:method} calibrates the evidence scale so that the composed
teacher's displacement from the language prior matches the scaled source KL
in \eqref{eq:match}. The visual residual is measured against the grounding
expert's own image-free prediction $q_V^0$ but applied to the language
expert's prior $q_L$. When the two priors coincide, $q_L=q_V^0$, exact
calibration gives $\alpha^\star=1$ and recovers $q_V^I$
\eqref{eq:app-self-prior} (Appendix~\ref{app:special-cases}). We use a
three-token example in which the LLM prior differs substantially from the
VLM's image-free prior to show why a single fixed coefficient cannot meet
the calibration reference across prefixes (Figure~\ref{fig:toy_calibration}).
Each row fixes a prefix. The grounding expert provides $q_V^0$ and $q_V^I$,
the language expert provides $q_L$, and the composed teacher is
$q_\alpha(v)\propto q_L(v)\exp\{\alpha r(v)\}$ as in \eqref{eq:family}.
We use $\rho=1$, so exact calibration requires
$\KL(q_{\alpha^\star}\|q_L)=\KL(q_V^I\|q_V^0)$.

\noindent\textbf{Reasoning prefix.}
The top row of Figure~\ref{fig:toy_calibration} shows a prefix at which
the image is nearly irrelevant. The grounding expert changes only at a
low-mass token, where $q_V(v_1)$ moves from $0.01$ to $0.02$, and its
argmax stays at $v_2$. The source KL is therefore small,
$\KL(q_V^I\|q_V^0)=0.004$. The residual $r(v_1)=\log 2$ is nevertheless
large, because it is a ratio. With $\alpha=1$, this ratio is applied to a
language prior that already assigns $0.35$ to $v_1$. The teacher's probability
of $v_1$ rises to $0.52$, which overturns the language expert's preferred
token $v_2$. The resulting displacement is $\KL(q_{\alpha=1}\|q_L)=0.062$,
or $15.7$ times the source KL. Calibration selects $\alpha^\star=0.26$, which
matches the source KL and preserves $v_2$.

\noindent\textbf{Grounding prefix.}
The bottom row of Figure~\ref{fig:toy_calibration} shows a prefix at which
the image determines the answer. The grounding expert moves from $v_2$ to
$v_1$, with $q_V(v_1)$ increasing from $0.20$ to $0.80$, which gives
$\KL(q_V^I\|q_V^0)=0.832$. The language expert cannot see the image and
assigns only $0.03$ to $v_1$. With $\alpha=1$, the teacher moves by only
$\KL(q_{\alpha=1}\|q_L)=0.546$, or $0.66$ times the source KL, and still
prefers $v_2$. Calibration selects $\alpha^\star=1.15$, and the teacher
follows the image.

\noindent\textbf{Discussion.}
The two prefixes require coefficients that differ by more than a factor of
four, so no single fixed value suits both. Raising the fixed coefficient to
$\alpha=2$ corrects the grounding prefix ($q_\alpha(v_1)=0.89$), but it
increases the displacement at the reasoning prefix to $60$ times the source KL
($q_\alpha(v_1)=0.69$). This pattern mirrors
Table~\ref{tab:calibration_ablation}, where larger fixed coefficients improve
perception but reduce text-only reasoning. Source-referenced calibration
avoids this trade-off by tying the displacement at each prefix to how much the
image changed the grounding expert itself.

\begin{figure}[H]
  \centering
  \includegraphics[width=\linewidth]{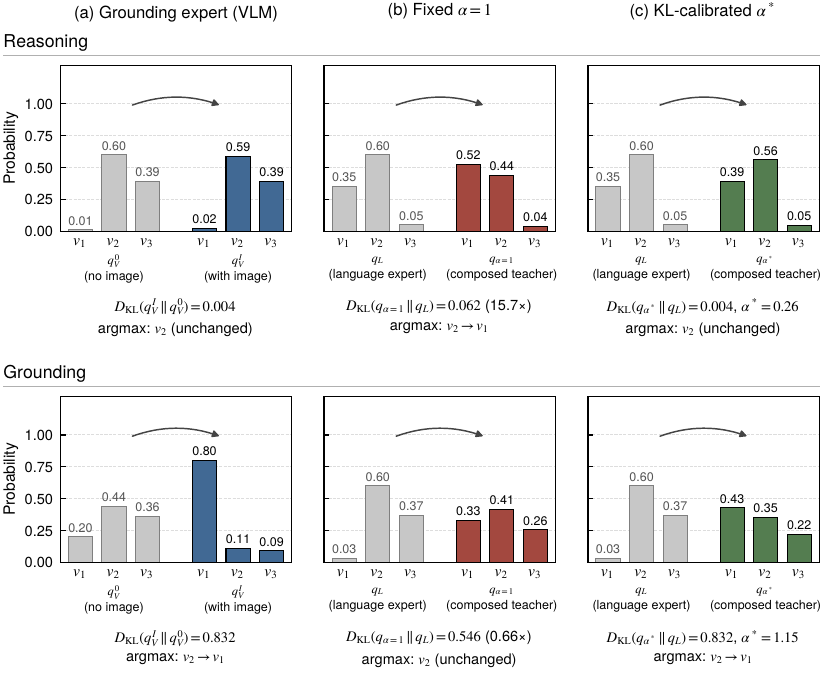}
\caption{\textbf{Illustrative example: a fixed coefficient cannot fit every
  prefix.} Gray bars show the next-token distribution before the update
  ($q_V^0$ or $q_L$), and colored bars show it after the update ($q_V^I$ or
  the composed teacher). \emph{Top:} at a reasoning prefix, the image barely
  changes the grounding expert ($\KL(q_V^I\|q_V^0)=0.004$). With $\alpha=1$,
  however, the teacher moves $15.7\times$ the source KL, and its argmax
  changes from $v_2$ to $v_1$. \emph{Bottom:} at a grounding prefix, the
  image determines the answer ($\KL(q_V^I\|q_V^0)=0.832$). With $\alpha=1$,
  the teacher reaches only $0.66\times$ the source KL and keeps $v_2$.
  KL calibration with $\rho=1$ selects $\alpha^\star=0.26$ and $1.15$,
  respectively, which match the source KL in both cases. Multiples in
  parentheses are relative to the source KL.}
  \label{fig:toy_calibration}
\end{figure}

\subsection{Fixed and Adaptive Calibration}
\label{app:controlled-ablations}
\label{app:calibration-results}

\noindent\textbf{Fixed-coefficient ablation.}
We compare adaptive calibration at $\rho=1$ with constant coefficients
$a\in\{1,2,4\}$ using the same student and the same 4B language and
grounding experts as in Table~\ref{tab:lego_settings}. The only change
is to replace the prefix-dependent coefficient $\widehat\alpha_{S,t}$
with $a$ at every response position:
\begin{equation}
q_{T,t}^{\mathrm{fixed}}
:=\operatorname{softmax}_{\V}\bigl(z_{L,t}+a u_t\bigr).
\label{eq:app-fixed-teacher}
\end{equation}
All runs retain the residual stabilization, teacher normalization,
student-top-16 update, and training and evaluation settings in
Appendices~\ref{app:implementation} and~\ref{app:training_settings}.
Only coefficient estimation is bypassed in the fixed variants.

\noindent\textbf{KL-budget scale.}
For the adaptive variants, we compare $\rho\in\{1,2,4\}$ while
retaining the remaining settings in Table~\ref{tab:lego_settings}.
All variants use the practical coefficient estimator in
\eqref{eq:practical-alpha}, with calibration support mass
$m=0.95$, variance regularizer $\epsilon=10^{-8}$, and coefficient
cap $\alpha_{\max}=6$. Changing $\rho$ rescales the source-KL budget
before coefficient estimation. Results for all reasoning and
perception benchmarks are reported in
Table~\ref{tab:calibration_ablation}.
Table~\ref{tab:calibration-seed-results} adds evaluation SDs for all
reasoning results in Table~\ref{tab:calibration_ablation}, using the
reporting protocol in Appendix~\ref{app:evaluation}.

\begingroup
\begin{table}[H]
\centering
\caption{\textbf{Calibration ablations with evaluation variability.}
All runs use the same 2B student and 4B language and grounding experts
for 200 training steps. Fixed composition uses constant $\alpha$,
while adaptive calibration uses KL-budget scale $\rho$.
Scores are mean $\pm$ SD over five evaluation sampling seeds.
Averages report means only.}
\label{tab:calibration-seed-results}
\footnotesize
\setlength{\tabcolsep}{3pt}
\renewcommand{\arraystretch}{1.12}
\renewcommand{\theadfont}{\footnotesize\bfseries}
\renewcommand{\theadgape}{\Gape[1pt]}
\resizebox{\linewidth}{!}{%
\begin{tabular}{@{}lcccccc@{}}
\toprule
& \multicolumn{3}{c}{\textbf{Fixed coefficient}} & \multicolumn{3}{c}{\textbf{Adaptive calibration}} \\
\cmidrule(lr){2-4}\cmidrule(lr){5-7}
\textbf{Benchmark} & $\alpha=1$ & $\alpha=2$ & $\alpha=4$ & $\rho=1$ & $\rho=2$ & $\rho=4$ \\
\midrule
\multicolumn{7}{c}{\textbf{Multimodal Reasoning}} \\
\midrule
MATH-Vision & $31.91\pm1.09$ & $31.25\pm2.58$ & $30.00\pm1.42$ & $33.36\pm2.56$ & $34.28\pm2.06$ & $34.01\pm0.76$ \\
MathVerse (Total) & $53.46\pm0.44$ & $54.29\pm0.50$ & $53.41\pm0.69$ & $56.96\pm0.20$ & $57.19\pm0.44$ & $57.51\pm0.51$ \\
MathVista & $64.58\pm0.72$ & $68.98\pm0.79$ & $69.36\pm0.33$ & $69.62\pm0.72$ & $70.58\pm1.02$ & $70.50\pm0.55$ \\
MMMU-Pro & $39.48\pm0.77$ & $38.87\pm0.66$ & $37.45\pm0.64$ & $40.98\pm1.14$ & $41.38\pm0.72$ & $41.36\pm0.61$ \\
RBench-M & $26.89\pm1.63$ & $24.03\pm0.98$ & $17.62\pm0.78$ & $29.17\pm1.20$ & $29.20\pm1.49$ & $28.39\pm2.51$ \\
OlympiadBench-MM & $38.27\pm2.03$ & $41.87\pm2.72$ & $41.73\pm3.70$ & $44.93\pm2.24$ & $45.07\pm3.85$ & $45.33\pm1.56$ \\
\midrule
\textbf{Multimodal Avg.} & 42.43 & 43.22 & 41.60 & 45.84 & 46.28 & 46.18 \\
\midrule
\multicolumn{7}{c}{\textbf{Text Reasoning}} \\
\midrule
MATH-500 & $82.24\pm1.10$ & $82.08\pm1.12$ & $81.52\pm1.54$ & $82.40\pm0.99$ & $82.36\pm0.84$ & $82.52\pm0.88$ \\
MathVerse (Text Only) & $59.06\pm0.72$ & $58.63\pm1.14$ & $57.94\pm0.68$ & $59.59\pm0.99$ & $60.23\pm0.40$ & $60.46\pm0.50$ \\
MMLU-Pro & $47.40\pm0.47$ & $47.16\pm0.25$ & $47.10\pm0.10$ & $47.15\pm0.05$ & $47.54\pm0.18$ & $47.38\pm0.26$ \\
RBench-T & $38.21\pm1.94$ & $35.52\pm0.70$ & $32.85\pm0.64$ & $38.59\pm1.19$ & $40.27\pm1.99$ & $38.08\pm0.59$ \\
OlympiadBench-Text & $52.20\pm1.33$ & $51.72\pm1.42$ & $51.19\pm1.20$ & $52.31\pm0.89$ & $53.00\pm0.77$ & $51.78\pm1.09$ \\
\midrule
\textbf{Text Avg.} & 55.82 & 55.02 & 54.12 & 56.01 & 56.68 & 56.04 \\
\bottomrule
\end{tabular}%
}
\end{table}
\endgroup

\subsection{Language-Only Supervision at Zero KL Budget}
\label{app:zero-budget-controls}

The dataset released by H-OPD~\citep{yin2026hopd} provides
image-and-question inputs for the VLM and
question-and-image-description inputs for the LLM.
We refer to these formats as \emph{multimodal data} and
\emph{text-only data}, respectively.

To examine the role of grounding supervision, we compare two
language-only controls at $\rho=0$, where the teacher reduces to
the language expert, $q_{T,t}=q_{L,t}$.
One control uses multimodal data, matching the data configuration
of the $\rho=1$ runs. The other uses text-only data, giving the LLM
teacher access to image descriptions.
Apart from $\rho$ and the data format, both controls follow the
$\rho=1$ training and evaluation settings, including the same
2B student initialization, 4B language expert, and 200 training steps.
Table~\ref{tab:zero-budget-controls} compares these controls with
the initial student and two $\rho=1$ configurations from
Table~\ref{tab:expert_combinations_full}.
Both configurations use multimodal data and the same 4B language
expert, with either the frozen initial student or the 4B VLM
providing the visual residual.

\begingroup
\begin{table}[H]
\centering
\caption{\textbf{Language-only supervision and expert composition.}
Base 2B is the initial student before distillation.
Student (2B) denotes its frozen copy used as the grounding expert.
Scores are accuracies (\%), reported as mean $\pm$ SD over five
evaluation sampling seeds. Benchmark averages report means only.}
\label{tab:zero-budget-controls}
\footnotesize
\setlength{\tabcolsep}{3pt}
\renewcommand{\arraystretch}{1.12}
\renewcommand{\theadfont}{\footnotesize\bfseries}
\renewcommand{\theadgape}{\Gape[1pt]}
\resizebox{\linewidth}{!}{%
\begin{tabular}{@{}lccccc@{}}
\toprule
& \thead{Base 2B}
& \multicolumn{4}{c}{\textbf{Distilled 2B students}} \\
\cmidrule(lr){3-6}
\textbf{Training data} & -- & \thead{Multimodal} & \thead{Text-only}
& \thead{Multimodal} & \thead{Multimodal} \\
Grounding expert & -- & -- & -- & Student (2B) & VLM 4B \\
KL-budget scale & -- & $\rho=0$ & $\rho=0$ & $\rho=1$ & $\rho=1$ \\
\midrule
\multicolumn{6}{c}{\textbf{Multimodal Reasoning}} \\
\midrule
MATH-Vision
& $27.96\pm1.68$ & $21.05\pm2.17$ & $33.75\pm1.80$
& $33.03\pm1.35$ & $33.36\pm2.56$ \\
MathVerse (Total)
& $51.51\pm0.48$ & $30.25\pm0.18$ & $54.86\pm0.53$
& $54.17\pm0.65$ & $56.96\pm0.20$ \\
MathVista
& $66.56\pm0.71$ & $32.22\pm0.74$ & $66.34\pm0.75$
& $69.42\pm1.13$ & $69.62\pm0.72$ \\
MMMU-Pro
& $37.40\pm0.66$ & $23.93\pm0.14$ & $39.58\pm0.57$
& $41.93\pm0.69$ & $40.98\pm1.14$ \\
RBench-M
& $19.61\pm0.95$ & $25.86\pm0.66$ & $26.95\pm0.65$
& $28.48\pm1.95$ & $29.17\pm1.20$ \\
OlympiadBench-MM
& $37.60\pm3.42$ & $32.80\pm1.28$ & $39.73\pm1.67$
& $43.07\pm3.55$ & $44.93\pm2.24$ \\
\midrule
\textbf{Multimodal Avg.} & 40.11 & 27.69 & 43.54 & 45.02 & 45.84 \\
\midrule
\multicolumn{6}{c}{\textbf{Text Reasoning}} \\
\midrule
MATH-500
& $80.44\pm0.95$ & $82.44\pm0.89$ & $82.08\pm0.97$
& $82.36\pm0.57$ & $82.40\pm0.99$ \\
MathVerse (Text Only)
& $57.18\pm1.01$ & $60.46\pm0.74$ & $59.57\pm0.86$
& $59.92\pm0.61$ & $59.59\pm0.99$ \\
MMLU-Pro
& $43.59\pm0.17$ & $48.11\pm0.33$ & $47.44\pm0.44$
& $47.51\pm0.35$ & $47.15\pm0.05$ \\
RBench-T
& $30.20\pm0.31$ & $38.21\pm1.51$ & $40.05\pm0.59$
& $38.46\pm1.30$ & $38.59\pm1.19$ \\
OlympiadBench-Text
& $48.04\pm1.42$ & $52.61\pm0.83$ & $52.73\pm0.58$
& $52.67\pm1.35$ & $52.31\pm0.89$ \\
\midrule
\textbf{Text Avg.} & 51.89 & 56.37 & 56.37 & 56.18 & 56.01 \\
\bottomrule
\end{tabular}%
}
\end{table}
\endgroup

These results suggest that text-only reasoning gains do not by
themselves ensure effective multimodal transfer.
The benefit of incorporating the frozen student's visual residual
supports a complementary role for image-conditioned supervision.
In this setting, factor composition can therefore benefit from
the student's existing visual signal without requiring a stronger
grounding model.

\subsection{Sensitivity to Calibration Support Mass}
\label{app:support-mass-ablation}

To assess sensitivity to the calibration support, we compare nucleus
masses $m\in\{0.95,0.99\}$ in the union construction of
\eqref{eq:union}. Both runs use the 4B language and grounding
experts, the same 2B student initialization, $\rho=1$, and 200 training
steps. All other training, teacher-construction, and evaluation settings
are held fixed. Only the mass threshold used to construct the
coefficient-estimation support $S$ is changed; the student top-16
support rule and full-vocabulary teacher normalization are unchanged.

Figure~\ref{fig:support-mass-training} shows the training diagnostics.
Both settings produce broadly similar decreases in training and top-$k$
OPD losses. The entropy-gap and support-mass diagnostics also follow
similar trajectories, with differences in the top-$k$ KL estimate
at several stages of training.

\begin{figure}[H]
\centering
\includegraphics[width=\linewidth]{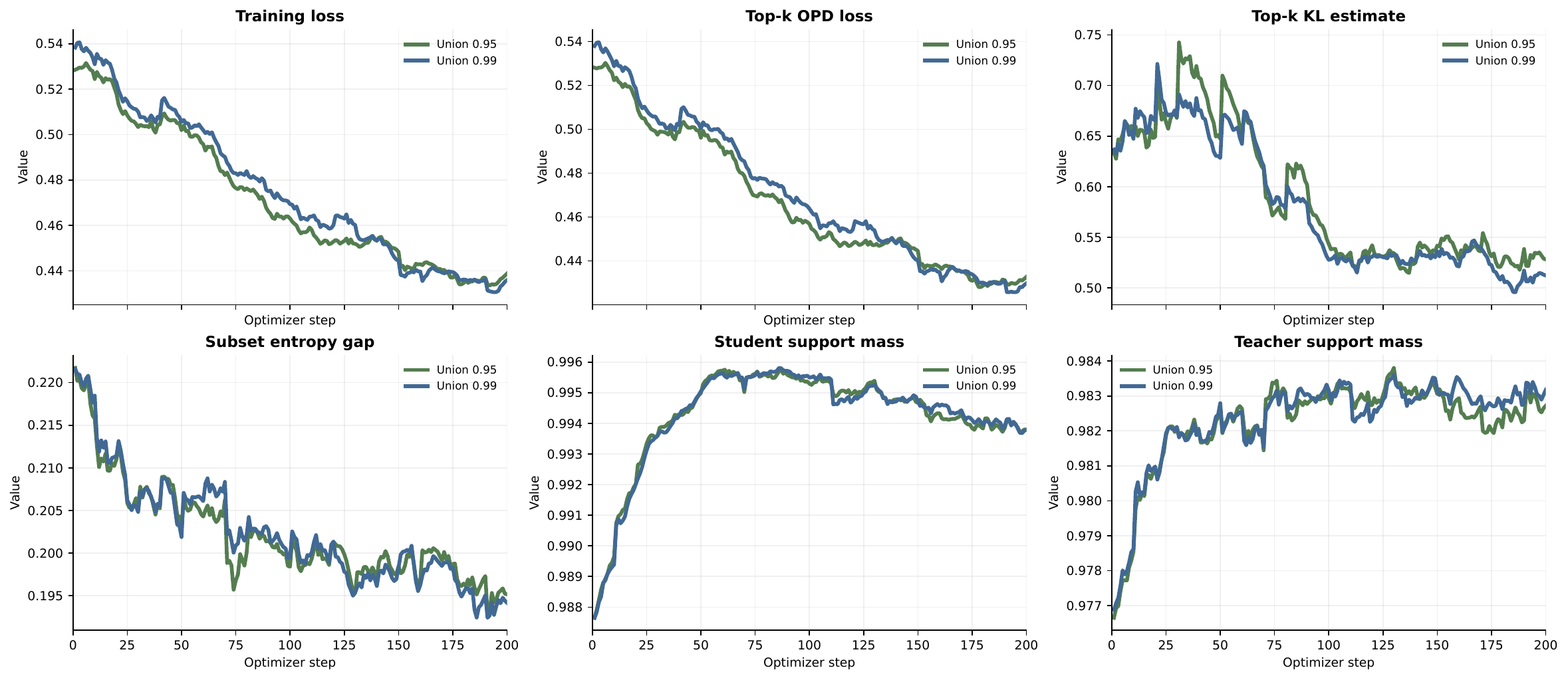}
\caption{\textbf{Training diagnostics for calibration support masses 0.95 and 0.99.}
The 4B--4B composition is trained at $\rho=1$ with all other settings
fixed. Panels compare training loss, top-$k$ OPD loss,
top-$k$ KL estimate, subset entropy gap,
student support mass, and teacher support mass over 200 optimizer steps.}
\label{fig:support-mass-training}
\end{figure}

\begin{table}[H]
\centering
\caption{\textbf{Effect of calibration support mass on final reasoning performance.}
Both runs use the same 2B student and 4B language and grounding experts
for 200 steps at $\rho=1$. Accuracies (\%) are mean $\pm$ SD across five evaluation
sampling seeds. Averages are unweighted means of the benchmark means.}
\label{tab:support-mass-results}
\small
\setlength{\tabcolsep}{5pt}
\renewcommand{\arraystretch}{1.12}
\begin{tabular*}{\linewidth}{@{\extracolsep{\fill}}lcc@{}}
\toprule
\textbf{Benchmark} & $m=0.95$ (default) & $m=0.99$ \\
\midrule
\multicolumn{3}{c}{\textbf{Multimodal Reasoning}} \\
\midrule
MATH-Vision & $33.36\pm2.56$ & $35.20\pm1.45$ \\
MathVerse (Total) & $56.96\pm0.20$ & $56.97\pm0.54$ \\
MathVista & $69.62\pm0.72$ & $69.62\pm0.51$ \\
MMMU-Pro & $40.98\pm1.14$ & $40.74\pm0.47$ \\
RBench-M & $29.17\pm1.20$ & $28.39\pm0.79$ \\
OlympiadBench-MM & $44.93\pm2.24$ & $44.67\pm2.45$ \\
\midrule
\textbf{Multimodal Avg.} & 45.84 & 45.93 \\
\midrule
\multicolumn{3}{c}{\textbf{Text Reasoning}} \\
\midrule
MATH-500 & $82.40\pm0.99$ & $81.80\pm0.65$ \\
MathVerse (Text Only) & $59.59\pm0.99$ & $59.80\pm0.91$ \\
MMLU-Pro & $47.15\pm0.05$ & $47.51\pm0.29$ \\
RBench-T & $38.59\pm1.19$ & $38.90\pm1.05$ \\
OlympiadBench-Text & $52.31\pm0.89$ & $52.11\pm0.71$ \\
\midrule
\textbf{Text Avg.} & 56.01 & 56.02 \\
\bottomrule
\end{tabular*}
\end{table}

Table~\ref{tab:support-mass-results} shows mixed changes across
individual benchmarks, with similar final averages: 45.84 versus 45.93
for multimodal reasoning and 56.01 versus 56.02 for text-only reasoning
at $m=0.95$ and $m=0.99$, respectively. Together with the training
curves, these results indicate limited sensitivity to increasing the
support mass from 0.95 to 0.99 in this configuration and support
retaining $m=0.95$ as the default.

\subsection{Approximation Accuracy of the Practical Coefficient}
\label{app:calibration-diagnostics}

Training approximates the exact KL matching in \eqref{eq:match} with
the second-order coefficient in \eqref{eq:local-coefficient}, computed
on a restricted support with a variance regularizer and a coefficient
cap as in \eqref{eq:practical-alpha}. We measure how closely this
practical coefficient matches a numerical solution of the KL-matching
equation on prefixes generated during training, in two settings.
Figure~\ref{fig:calibration_accuracy} in the main text composes the 4B
language and grounding experts, with the language expert's distribution
as the receiving prior $q_{L,t}$. Figure~\ref{fig:calibration_recon}
uses self-prior recomposition, which instead sets $q_{L,t}=q_{V,t}^{0}$
with the grounding VLM's own image-free prediction.

\noindent\textbf{Setup.}
Both settings use $\rho=1$ and the same calibration support and residual
stabilization as in student training. At each prefix, we construct the
stabilized residual $u_t$ as in Appendix~\ref{app:residual-stabilization}
and the support $S_t$ as in \eqref{eq:union} with $m=0.95$. For expert
composition, $S_t$ is the union of the language, image-conditioned, and
image-free nuclei. For self-prior recomposition, the language and
image-free branches coincide, leaving two distinct nuclei. All
distributions are renormalized on $S_t$ before computing KL. The source
KL uses the original image-conditioned and image-free predictions, and
the teacher KL measures the displacement of the composed teacher from
the receiving prior on $S_t$. Distillation itself uses the
full-vocabulary teacher.

\noindent\textbf{Compared coefficients.}
The practical coefficient $\widehat\alpha_{S,t}$ is computed by
\eqref{eq:practical-alpha}, including the variance regularizer and
coefficient cap. The numerical reference solves for the coefficient at
which the teacher KL equals the source KL, using the same residual,
receiving prior, and support. In Figures~\ref{fig:calibration_accuracy}
and~\ref{fig:calibration_recon}, panel (a) compares the two coefficients,
and panel (b) compares the resulting teacher KL with the source KL,
where the diagonal marks exact matching.

\noindent\textbf{Relation to exact self-prior recovery.}
With the raw residual, composing $q_{V,t}^{0}$ at $\alpha=1$ recovers
$q_{V,t}^{I}$, also after renormalization on a common support.
Residual stabilization changes the composition path, so the numerical
reference in Figure~\ref{fig:calibration_recon} can differ from one.
It therefore serves as a reference for the practical self-prior path.
Appendix~\ref{app:recomposed-training} evaluates the effect of this
practical construction on distillation.

\noindent\textbf{Prefixes and aggregation.}
The diagnostic uses 24,576 prefixes from 768 rollouts on 256
MMMU~\citep{yue2024mmmu} questions, generated by student checkpoints at
steps 0, 100, and 200. MMMU is separate from the MMMU-Pro benchmark used
for evaluation. The practical and numerical coefficients are evaluated on
the same prefixes. The figures report means within source-KL bins, with
pointwise 95\% confidence intervals from question-level bootstrap
resampling. Teacher-KL curves average per-prefix divergences.
Binned means can agree even when individual prefixes deviate, and the
upward deviations at larger source KL reflect finite-update error in the
practical approximation.

\begin{figure}[H]
\centering
\setlength{\abovecaptionskip}{3pt}
\begin{minipage}[t]{0.49\linewidth}
\centering
\includegraphics[width=\linewidth]{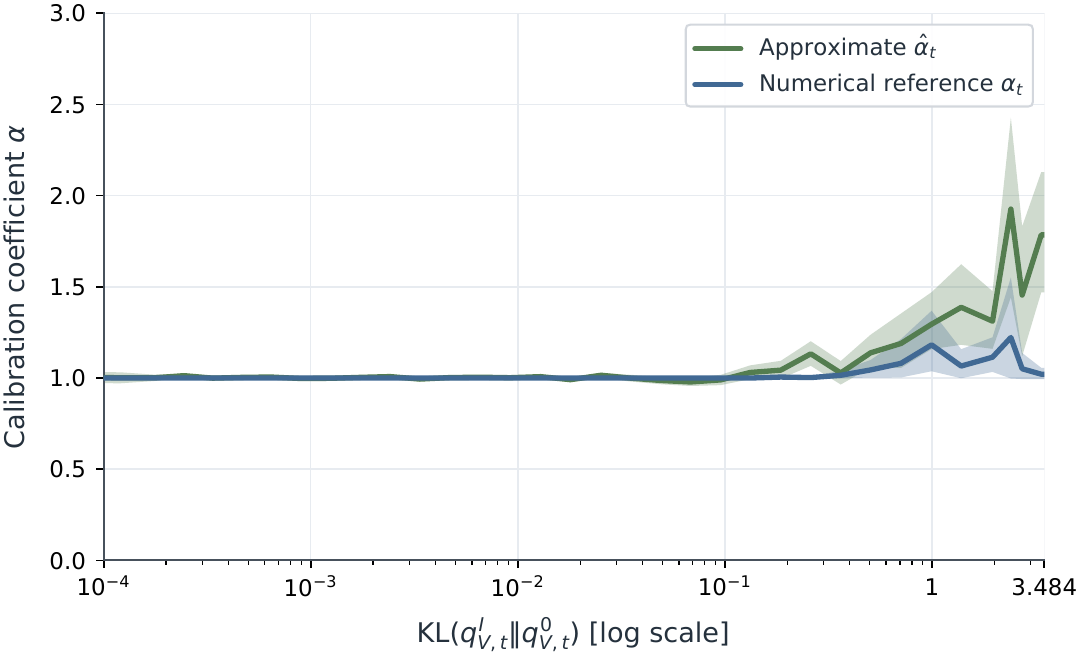}
\par\vspace{-3pt}
{\small (a) Calibration coefficient}
\end{minipage}
\hfill
\begin{minipage}[t]{0.49\linewidth}
\centering
\includegraphics[width=\linewidth]{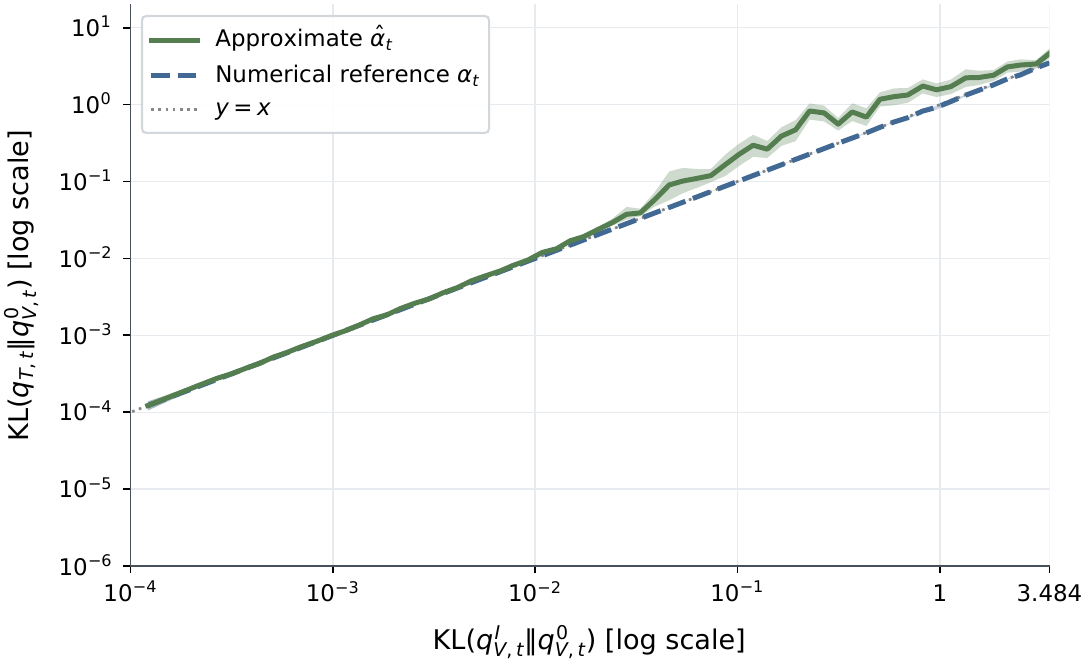}
\par\vspace{-3pt}
{\small (b) Induced teacher KL}
\end{minipage}
\caption{\textbf{Local approximation error in self-prior recomposition.}
The 4B VLM supplies both its own image-free prior and visual residual,
with $\rho=1$. Panels compare calibration coefficients (a) and teacher KL
relative to the VLM's image-free prior (b).
Curves show means within source-KL bins. The dotted $y=x$ line denotes
source-KL matching. Shaded bands indicate question-bootstrap 95\% confidence intervals.}
\label{fig:calibration_recon}
\end{figure}

\subsection{Distillation with a Recomposed VLM Teacher}
\label{app:recomposed-training}
\label{sec:approximation-analysis}

We define a \emph{recomposed teacher} by decomposing a single VLM into
its own image-free prior and visual residual and recombining them with
adaptive calibration at $\rho=1$. The receiving prior is the VLM's
image-free distribution, $q_{L,t}:=q_{V,t}^{0}$, and its
image-conditioned and image-free predictions supply the visual residual.
The recomposed teacher thus replaces the external language prior of the
main composition with the grounding expert's own prior. To validate the
practical implementation, we compare a student trained with this teacher
against standard OPD, which distills directly from the frozen
Qwen3-VL-4B-Instruct teacher.

All other settings match the main experiment, including the student
initialization, training data and budget, evaluation protocol, and the
practical construction in Appendix~\ref{app:implementation}
(residual stabilization, $m=0.95$, $\epsilon=10^{-8}$, and
$\alpha_{\max}=6$). The recomposed teacher is therefore
\begin{equation}
q_{T,t}^{\mathrm{recon}}
:=\operatorname{softmax}_{\V}
\bigl(z_{V,t}^{0}+\widehat\alpha_{S,t}u_t\bigr),
\label{eq:app-recomposed-teacher}
\end{equation}
where the coefficient statistics use $q_{V,t}^{0}$ as the receiving
prior. The direct-teacher condition uses $q_{V,t}^{I}$ as its target.

For the raw residual and exact calibration, this self-prior construction
recovers the original image-conditioned teacher at $\rho=1$
(Appendix~\ref{app:special-cases}). The training comparison instead
evaluates the complete practical recomposition, including coefficient
approximation and stabilization. Its performance difference from
standard OPD therefore reflects their combined effect.

\noindent\textbf{Distillation results.}
Table~\ref{tab:approximation_sanity} reports the complete reasoning results.
The recomposed teacher obtains similar aggregate performance to direct
VLM supervision (40.18 vs.\ 40.37 for multimodal reasoning and
51.18 vs.\ 51.72 for text reasoning), with mixed changes across benchmarks.
This comparison supports the usefulness of the practical teacher for
distillation, while not establishing equivalence or attributing individual
gains to larger visual updates. Coefficient approximation and residual
stabilization are evaluated together, so their separate training effects
cannot be isolated by this comparison.

\begingroup
\begin{table}[H]
\centering
\caption{\textbf{Practical recomposition: downstream distillation.}
Both conditions train the same 2B student for 200 steps using the 4B VLM.
Recomposition uses the VLM's own image-free prior, stabilized visual
residual, and approximate calibration at $\rho=1$.
Scores are mean $\pm$ SD over five evaluation sampling seeds.
$\Delta$ is recomposed minus direct-supervision accuracy in pp.
Averages report means only.}
\label{tab:approximation_sanity}
\footnotesize
\setlength{\tabcolsep}{3pt}
\renewcommand{\arraystretch}{1.12}
\renewcommand{\theadfont}{\footnotesize\bfseries}
\renewcommand{\theadgape}{\Gape[1pt]}
\begin{tabular*}{\linewidth}{@{\extracolsep{\fill}}lccr@{}}
\toprule
\textbf{Benchmark} & \thead{Direct VLM\\supervision} & \thead{Recomposed\\teacher} & $\Delta$ (pp) \\
\midrule
\multicolumn{4}{c}{\textbf{Multimodal Reasoning}} \\
\midrule
MATH-Vision & $27.43\pm0.89$ & $28.62\pm1.73$ & $+1.19$ \\
MathVerse (Total) & $52.70\pm0.76$ & $53.43\pm0.38$ & $+0.73$ \\
MathVista & $68.60\pm0.39$ & $68.34\pm1.28$ & $-0.26$ \\
MMMU-Pro & $36.98\pm0.89$ & $36.72\pm0.51$ & $-0.26$ \\
RBench-M & $16.90\pm0.52$ & $13.86\pm1.27$ & $-3.04$ \\
OlympiadBench-MM & $39.60\pm0.37$ & $40.13\pm2.72$ & $+0.53$ \\
\midrule
\textbf{Multimodal Avg.} & 40.37 & 40.18 & $-0.19$ \\
\midrule
\multicolumn{4}{c}{\textbf{Text Reasoning}} \\
\midrule
MATH-500 & $78.72\pm0.82$ & $78.72\pm1.85$ & $+0.00$ \\
MathVerse (Text Only) & $57.44\pm0.82$ & $56.80\pm1.19$ & $-0.64$ \\
MMLU-Pro & $45.52\pm0.16$ & $45.62\pm0.33$ & $+0.10$ \\
RBench-T & $28.59\pm1.06$ & $28.01\pm1.05$ & $-0.58$ \\
OlympiadBench-Text & $48.31\pm0.70$ & $46.77\pm1.98$ & $-1.54$ \\
\midrule
\textbf{Text Avg.} & 51.72 & 51.18 & $-0.54$ \\
\bottomrule
\end{tabular*}
\end{table}
\endgroup

\section{Supplementary Experiments}
\label{app:supplementary-experiments}

\subsection{Main Results and Decoding Variants}
\label{app:main-results}

\subsubsection{Sampled Decoding}
\label{app:seed-results}

Table~\ref{tab:main-seed-results} adds evaluation SDs to the reasoning
results in Table~\ref{tab:qwen3vl_main}. Scores are accuracies (\%),
reported using the five-seed evaluation protocol in
Appendix~\ref{app:evaluation}.
For LEGO-OPD, this table reports $\rho=1$; results at $\rho=2$
with SDs are provided in Appendix~\ref{app:calibration-results}.

\begingroup
\begin{table}[H]
\centering
\caption{\textbf{Main reasoning results with evaluation variability.}
Scores are accuracies (\%, mean $\pm$ SD across five sampling seeds).
Base 2B is the initial student and Base 4B is the original VLM reference.
All distilled models are 2B students evaluated at step 200.
Averages report means only.}
\label{tab:main-seed-results}
\footnotesize
\setlength{\tabcolsep}{3pt}
\renewcommand{\arraystretch}{1.12}
\renewcommand{\theadfont}{\footnotesize\bfseries}
\renewcommand{\theadgape}{\Gape[1pt]}
\resizebox{\linewidth}{!}{%
\begin{tabular}{@{}lccccccc@{}}
\toprule
\textbf{Benchmark} & \thead{Base\\4B} & \thead{Base\\2B} & \thead{Standard\\OPD} & \thead{VA-OPD} & \thead{M-OPD} & \thead{H-OPD} & \thead{LEGO-OPD\\$\rho=1$} \\
\midrule
\multicolumn{8}{c}{\textbf{Multimodal Reasoning}} \\
\midrule
MATH-Vision & $39.74\pm2.25$ & $27.96\pm1.68$ & $27.43\pm0.89$ & $25.33\pm1.12$ & $24.41\pm1.73$ & $32.37\pm1.83$ & $33.36\pm2.56$ \\
MathVerse (Total) & $66.70\pm0.51$ & $51.51\pm0.48$ & $52.70\pm0.76$ & $51.66\pm0.36$ & $51.28\pm0.53$ & $55.41\pm0.65$ & $56.96\pm0.20$ \\
MathVista & $76.26\pm0.59$ & $66.56\pm0.71$ & $68.60\pm0.39$ & $67.36\pm0.93$ & $66.42\pm1.00$ & $67.70\pm0.79$ & $69.62\pm0.72$ \\
MMMU-Pro & $50.29\pm0.85$ & $37.40\pm0.66$ & $36.98\pm0.89$ & $35.11\pm1.16$ & $35.06\pm0.85$ & $39.36\pm0.69$ & $40.98\pm1.14$ \\
RBench-M & $32.81\pm1.34$ & $19.61\pm0.95$ & $16.90\pm0.52$ & $16.81\pm0.62$ & $15.67\pm1.30$ & $26.92\pm1.59$ & $29.17\pm1.20$ \\
OlympiadBench-MM & $63.20\pm2.02$ & $37.60\pm3.42$ & $39.60\pm0.37$ & $37.33\pm1.70$ & $36.53\pm1.66$ & $39.07\pm3.08$ & $44.93\pm2.24$ \\
\midrule
\textbf{Multimodal Avg.} & 54.83 & 40.11 & 40.37 & 38.93 & 38.23 & 43.47 & 45.84 \\
\midrule
\multicolumn{8}{c}{\textbf{Text Reasoning}} \\
\midrule
MATH-500 & $93.04\pm0.55$ & $80.44\pm0.95$ & $78.72\pm0.82$ & $80.20\pm0.69$ & $80.56\pm0.77$ & $81.88\pm1.08$ & $82.40\pm0.99$ \\
MathVerse (Text Only) & $66.80\pm0.17$ & $57.18\pm1.01$ & $57.44\pm0.82$ & $56.65\pm1.01$ & $56.19\pm0.79$ & $59.44\pm0.39$ & $59.59\pm0.99$ \\
MMLU-Pro & $62.90\pm0.14$ & $43.59\pm0.17$ & $45.52\pm0.16$ & $45.09\pm0.09$ & $44.98\pm0.30$ & $46.94\pm0.17$ & $47.15\pm0.05$ \\
RBench-T & $51.92\pm0.82$ & $30.20\pm0.31$ & $28.59\pm1.06$ & $28.76\pm0.52$ & $28.28\pm1.01$ & $37.17\pm1.12$ & $38.59\pm1.19$ \\
OlympiadBench-Text & $69.41\pm0.95$ & $48.04\pm1.42$ & $48.31\pm0.70$ & $48.99\pm0.73$ & $47.77\pm1.13$ & $52.40\pm0.95$ & $52.31\pm0.89$ \\
\midrule
\textbf{Text Avg.} & 68.81 & 51.89 & 51.72 & 51.94 & 51.56 & 55.57 & 56.01 \\
\bottomrule
\end{tabular}%
}
\end{table}
\endgroup

\subsubsection{Greedy Decoding}
\label{app:greedy-results}

Table~\ref{tab:greedy-results} reports greedy-decoding results for the
initial 2B student, the 4B VLM reference, and the distilled 2B students
at step 200. All evaluation settings are identical to those used for
sampling-based evaluation, except for the sampling parameters.
Standard OPD denotes distillation from a single VLM.
LEGO-OPD uses the 4B language and grounding experts at $\rho\in\{1,2\}$.
Under greedy decoding, LEGO-OPD at $\rho=2$ achieves the highest
multimodal and text reasoning averages among the evaluated 2B students.

\begingroup
\begin{table}[H]
\centering
\caption{\textbf{Greedy decoding results across reasoning benchmarks.}
Scores are accuracies (\%). Base 2B and Base 4B are the original
Qwen3-VL-Instruct models without distillation. All distilled students
are evaluated at step 200. Averages are unweighted means over the six
multimodal or five text benchmarks. Bold indicates the best score
among the 2B students, excluding the 4B reference.}
\label{tab:greedy-results}
\footnotesize
\setlength{\tabcolsep}{2pt}
\renewcommand{\arraystretch}{1.08}
\renewcommand{\theadfont}{\scriptsize\bfseries}
\renewcommand{\theadgape}{\Gape[1pt]}
\resizebox{\linewidth}{!}{%
\begin{tabular*}{\textwidth}{@{\extracolsep{\fill}}l*{8}{c}@{}}
\toprule
& \thead{4B\\reference} & \thead{Initial\\student} & \multicolumn{4}{c}{\textbf{Distillation baselines}} & \multicolumn{2}{c}{\textbf{LEGO-OPD}} \\
\cmidrule(lr){2-2}\cmidrule(lr){3-3}\cmidrule(lr){4-7}\cmidrule(lr){8-9}
\textbf{Benchmark} & \thead{Base 4B} & \thead{Base 2B} & \thead{Standard\\OPD} & \thead{VA-OPD} & \thead{M-OPD} & \thead{H-OPD} & $\rho=1$ & $\rho=2$ \\
\midrule
Step & -- & -- & 200 & 200 & 200 & 200 & 200 & 200 \\

\midrule
\multicolumn{9}{c}{\textbf{Multimodal Reasoning}} \\
\midrule
MATH-Vision & 31.91 & 8.88 & 13.82 & 11.84 & 11.18 & 19.74 & 25.66 & \textbf{28.95} \\
MathVerse (Total) & 62.77 & 38.40 & 39.80 & 40.18 & 39.57 & 47.11 & \textbf{53.78} & 53.43 \\
MathVista & 75.10 & 58.70 & 60.20 & 60.70 & 62.50 & 65.30 & 68.30 & \textbf{70.60} \\
MMMU-Pro & 44.22 & 27.17 & 24.80 & 27.40 & 26.42 & 32.02 & 36.99 & \textbf{37.34} \\
RBench-M & 14.14 & 2.41 & 2.86 & 2.41 & 3.16 & 7.37 & \textbf{18.35} & 15.79 \\
OlympiadBench-MM & 64.00 & 26.00 & 32.00 & 29.33 & 28.67 & \textbf{41.33} & 40.00 & 40.00 \\
\midrule
\textbf{Multimodal Avg.} & 48.69 & 26.93 & 28.91 & 28.64 & 28.58 & 35.48 & 40.51 & \textbf{41.02} \\
\midrule
\multicolumn{9}{c}{\textbf{Text Reasoning}} \\
\midrule
MATH-500 & 90.40 & 64.40 & 71.60 & 70.20 & 72.60 & 79.00 & 80.40 & \textbf{82.00} \\
MathVerse (Text Only) & 63.07 & 38.71 & 45.81 & 45.81 & 43.53 & 53.05 & 57.23 & \textbf{58.50} \\
MMLU-Pro & 38.75 & 43.26 & 24.72 & 38.22 & 38.11 & 36.69 & 44.82 & \textbf{46.07} \\
RBench-T & 33.27 & 6.58 & 9.14 & 10.42 & 9.41 & 20.66 & 28.43 & \textbf{29.16} \\
OlympiadBench-Text & 63.95 & 33.68 & 40.80 & 39.61 & 39.91 & 46.88 & 48.52 & \textbf{48.66} \\
\midrule
\textbf{Text Avg.} & 57.89 & 37.33 & 38.41 & 40.85 & 40.71 & 47.26 & 51.88 & \textbf{52.88} \\
\bottomrule
\end{tabular*}%
}
\end{table}
\endgroup

\subsection{Expert Configurations and Factor Contributions}
\label{app:additional-experts}

\noindent\textbf{Models and comparison settings.}
LLM 1.7B and LLM 4B denote Qwen3-1.7B and Qwen3-4B-Instruct-2507.
VLM 2B/4B/8B denote the corresponding Qwen3-VL-Instruct models.
All combinations train the same 2B student for 200 steps with $\rho=1$
and the practical settings in Appendix~\ref{app:implementation}.
A VLM used as the language expert receives no image.
The grounding expert supplies both image-conditioned and image-free
predictions for the visual residual.
Table~\ref{tab:expert-seed-results} provides the full reasoning results
with evaluation SDs for Tables~\ref{tab:expert_combinations_full}
and~\ref{tab:expert_contributions}. Shared configurations are reported
once. The 4B VLM self-recomposition condition is reported with the calibration analysis in Appendix~\ref{app:recomposed-training}.

\noindent\textbf{Model selection.}
The 1.7B language model and 2B/8B grounding models vary one expert while
the other remains fixed at 4B. Checkpoint generation and inference mode
also affect expert strength: Qwen3-4B-Instruct-2507 reports higher
non-thinking AIME25 and ZebraLogic scores than the original Qwen3-8B
\citep{yang2025qwen3,qwen2025qwen3_4b_instruct2507}.
Thus, a larger parameter count alone does not establish that the original
8B language model is a stronger replacement for the updated 4B checkpoint.

\noindent\textbf{Expert-factor substitution.}
To examine the contributions of the two factor sources, we replace
either the grounding expert or the language prior with the corresponding
factor from the frozen initial student, as described in
Section~\ref{sec:expert-contributions}. Grounding-expert substitution
uses the frozen student's image-conditioned and image-free predictions
to construct the visual residual. Language-prior substitution instead
uses the frozen student's image-free prediction as the receiving prior.
The remaining expert is retained, and the calibration statistics are
computed from the resulting factor sources.

\noindent\textbf{Benchmark subset for the compact factor comparison.}
\label{app:ablation-benchmarks}
The compact comparison in Table~\ref{tab:expert_contributions} uses
MATH-Vision, MathVerse, and OlympiadBench for multimodal evaluation.
MATH-Vision represents visual mathematical reasoning and is also
used in the baseline work. MathVerse provides fine-grained input
variants, while OlympiadBench adds challenging olympiad-level reasoning.
For text-only evaluation, we use the corresponding MathVerse and
OlympiadBench benchmarks to compare reasoning across modalities
within the same benchmark families.

\begingroup
\begin{table}[H]
\centering
\caption{\textbf{Expert combinations with evaluation variability.}
Column headers list the language expert above the grounding expert.
The shaded Base column uses the 4B language and grounding experts.
VLM 2B is the frozen initial student and receives no image when used
as the language expert. Scores are mean $\pm$ SD over five evaluation
sampling seeds. Averages report means only.}
\label{tab:expert-seed-results}
\footnotesize
\setlength{\tabcolsep}{3pt}
\renewcommand{\arraystretch}{1.12}
\renewcommand{\theadfont}{\footnotesize\bfseries}
\renewcommand{\theadgape}{\Gape[1pt]}
\resizebox{0.9\linewidth}{!}{%
\begin{tabular}{@{}lcc>{\columncolor{black!6}}ccc@{}}
\toprule
& \multicolumn{2}{c}{\textbf{Language expert swap}} & \textbf{Base} & \multicolumn{2}{c}{\textbf{Grounding expert swap}} \\
\cmidrule(lr){2-3}\cmidrule(lr){5-6}
\textbf{Benchmark} & \thead{LLM 1.7B\\VLM 4B} & \thead{VLM 2B\\VLM 4B} & \thead{LLM 4B\\VLM 4B} & \thead{LLM 4B\\VLM 2B} & \thead{LLM 4B\\VLM 8B} \\
\midrule
\multicolumn{6}{c}{\textbf{Multimodal Reasoning}} \\
\midrule
MATH-Vision & $28.36\pm2.84$ & $26.58\pm1.20$ & $33.36\pm2.56$ & $33.03\pm1.35$ & $34.34\pm1.77$ \\
MathVerse (Total) & $43.89\pm0.27$ & $52.69\pm0.54$ & $56.96\pm0.20$ & $54.17\pm0.65$ & $57.81\pm0.40$ \\
MathVista & $64.96\pm0.89$ & $68.20\pm0.47$ & $69.62\pm0.72$ & $69.42\pm1.13$ & $70.24\pm0.56$ \\
MMMU-Pro & $35.11\pm0.63$ & $36.36\pm0.77$ & $40.98\pm1.14$ & $41.93\pm0.69$ & $41.70\pm1.03$ \\
RBench-M & $22.71\pm1.48$ & $17.53\pm0.81$ & $29.17\pm1.20$ & $28.48\pm1.95$ & $28.90\pm1.40$ \\
OlympiadBench-MM & $36.00\pm2.21$ & $37.20\pm3.25$ & $44.93\pm2.24$ & $43.07\pm3.55$ & $45.20\pm1.28$ \\
\midrule
\textbf{Multimodal Avg.} & 38.51 & 39.76 & 45.84 & 45.02 & 46.37 \\
\midrule
\multicolumn{6}{c}{\textbf{Text Reasoning}} \\
\midrule
MATH-500 & $81.52\pm0.89$ & $79.60\pm0.85$ & $82.40\pm0.99$ & $82.36\pm0.57$ & $84.08\pm1.81$ \\
MathVerse (Text Only) & $55.03\pm1.84$ & $56.98\pm0.68$ & $59.59\pm0.99$ & $59.92\pm0.61$ & $59.59\pm0.95$ \\
MMLU-Pro & $43.14\pm0.41$ & $44.91\pm0.18$ & $47.15\pm0.05$ & $47.51\pm0.35$ & $46.92\pm0.31$ \\
RBench-T & $33.05\pm1.36$ & $29.87\pm1.33$ & $38.59\pm1.19$ & $38.46\pm1.30$ & $37.61\pm0.90$ \\
OlympiadBench-Text & $47.57\pm1.53$ & $48.07\pm0.87$ & $52.31\pm0.89$ & $52.67\pm1.35$ & $53.06\pm0.99$ \\
\midrule
\textbf{Text Avg.} & 52.06 & 51.89 & 56.01 & 56.18 & 56.25 \\
\bottomrule
\end{tabular}%
}
\end{table}
\endgroup

\subsection{Decomposed OPD Reproduction Attempt}
\label{app:decomposed-reproduction}

\noindent\textbf{Implementation and results.}
At the time of our experiments, no official implementation of
Decomposed OPD~\citep{yoon2026decomposedopd} was available.
We therefore implemented the method based on the paper.
However, we were unable to reproduce its reported improvements
in our experimental setting. We report this attempt separately
from the main baseline comparisons.

We use Qwen3-VL-2B-Instruct as the student and
Qwen3-VL-4B-Instruct as the teacher, with the training data and
evaluation protocol described in Section~\ref{sec:experimental-setup}.
The teacher is evaluated with and without the image.
Our implementation includes the standard distillation,
visual-steering, and language-preservation terms, with
$\gamma=2.0$, $\lambda=0.01$, and $\eta=0.41$.
The language-preservation mask uses the per-response 70th
percentile of the teacher visual-dependence score.
All distillation terms are computed over the full vocabulary,
whereas our Standard OPD baseline uses the student-top-16 surrogate.

Table~\ref{tab:decomposed-reproduction} reports the results.
Our implementation obtains lower average scores than Standard OPD
on both the multimodal and text-only benchmarks.
These results describe our implementation under the adapted setup
and do not establish the performance of the original method
under its original experimental conditions.

\begin{table}[H]
\centering
\caption{
Results of our Decomposed OPD reproduction attempt at step 200.
Scores are accuracies (\%). Reasoning scores report mean $\pm$ SD.
Perception scores use greedy decoding. Benchmark averages report means
only, computed as unweighted means over six multimodal or five text benchmarks.
Decomposed OPD denotes our paper-based implementation.
}
\label{tab:decomposed-reproduction}
\small
\setlength{\tabcolsep}{8pt}
\begin{tabular*}{\linewidth}{@{\extracolsep{\fill}}lcc@{}}
\toprule
Benchmark
& Standard OPD
& \makecell{Decomposed OPD\\(our implementation)} \\
\midrule
\multicolumn{3}{c}{\textbf{Multimodal reasoning}} \\
\midrule
MATH-Vision           & $27.43\pm0.89$ & $25.59 \pm 1.82$ \\
MathVerse (Total)     & $52.70\pm0.76$ & $50.80 \pm 0.30$ \\
MathVista            & $68.60\pm0.39$ & $63.94 \pm 1.28$ \\
MMMU-Pro             & $36.98\pm0.89$ & $34.52 \pm 0.56$ \\
RBench-M             & $16.90\pm0.52$ & $10.20 \pm 0.91$ \\
OlympiadBench (MM)   & $39.60\pm0.37$ & $35.60 \pm 1.74$ \\
\midrule
Multimodal reasoning Avg. & 40.37 & 36.78 \\
\midrule
\multicolumn{3}{c}{\textbf{Text-only reasoning}} \\
\midrule
MATH-500              & $78.72\pm0.82$ & $79.36 \pm 1.23$ \\
MathVerse (Text Only) & $57.44\pm0.82$ & $56.24 \pm 1.23$ \\
MMLU-Pro              & $45.52\pm0.16$ & $43.64 \pm 0.23$ \\
RBench-T              & $28.59\pm1.06$ & $22.23 \pm 1.07$ \\
OlympiadBench (Text)  & $48.31\pm0.70$ & $46.56 \pm 0.68$ \\
\midrule
Text-only reasoning Avg. & 51.72 & 49.61 \\
\midrule
\multicolumn{3}{c}{\textbf{Visual perception}} \\
\midrule
BLINK         & 52.13 & 51.60 \\
QBench1,2 & 69.54 & 69.18 \\
\bottomrule
\end{tabular*}
\end{table}

The relation between its auxiliary visual objective and LEGO-OPD's
loss decomposition is derived in Appendix~\ref{app:loss-decomposition}.

\section{Limitations}
\label{sec:limitations}

Although LEGO-OPD calibrates how strongly visual evidence modifies
the language prior, it does not determine whether that evidence is
reliable. Source KL measures the magnitude of the grounding expert's
visual update rather than its correctness, so calibration can amplify
erroneous evidence.

Beyond calibration, the visual residual represents a conditional
model contrast rather than an isolated visual factor, because
generated prefixes may already contain information inferred from
the image. In addition, the current implementation requires a shared
tokenization and output vocabulary, restricting composition to
compatible models within the same family.

\end{document}